\documentclass{article}

\usepackage{microtype}
\usepackage{graphicx}
\usepackage{subfigure}
\usepackage{booktabs} 
\usepackage{enumitem}
\PassOptionsToPackage{hyphens}{url}\usepackage[breaklinks=true]{hyperref}

\usepackage[accepted]{icml2022}

\usepackage{amsmath}
\usepackage{amssymb}
\usepackage{mathtools, cuted}
\usepackage{amsthm}
\usepackage[framemethod=tikz]{mdframed}
\usepackage[capitalize,noabbrev]{cleveref}

\theoremstyle{plain}

\newmdtheoremenv[
  linecolor=cyan,
  roundcorner=5pt,
  linewidth=1.5pt,
]{thm}{Theorem}

\newtheorem{lemma}[thm]{Lemma}

\newtheorem{prop}[thm]{Proposition}

\theoremstyle{definition}
\newmdtheoremenv[
  hidealllines=true,
  leftline=true,
  innerleftmargin=10pt,
  innerrightmargin=10pt,
  innertopmargin=0pt,
]{definition}{Definition}

\theoremstyle{remark}

\usepackage[textsize=tiny]{todonotes}

\newcommand{\lo}{\mathcal{L}_{L2O}}
\newcommand{\Ltask}{\mathcal{F}}
\newcommand{\N}{\mathbb{N}}
\newcommand{\R}{\mathbb{R}}
\newcommand{\Esp}{\mathbb{E}}

\icmltitlerunning{A Simple Guard for Learned Optimizers}

\begin{document}

\twocolumn[
\icmltitle{A Simple Guard for Learned Optimizers}
\icmlsetsymbol{equal}{*}

\begin{icmlauthorlist}
\icmlauthor{Jaroslav Vítků}{equal,gai}
\icmlauthor{Isabeau Prémont-Schwarz}{equal,gai}
\icmlauthor{Jan Feyereisl}{gai}

\end{icmlauthorlist}

\icmlaffiliation{gai}{Good AI, Prague, Czechia}

\icmlcorrespondingauthor{Isabeau Prémont-Schwarz}{premont-schwarz@goodai.com}
\icmlcorrespondingauthor{Jaroslav Vítků}{jaroslav.vitku@goodai.com}

\icmlkeywords{Optimization, Learned Optimizers, Convergence Guarantees, Hybrid Optimizers, Guarded Optimizers}

\vskip 0.3in
]


\printAffiliationsAndNotice{\icmlEqualContribution} 
\begin{abstract}
If the trend of learned components eventually outperforming their hand-crafted version continues, learned optimizers will eventually outperform hand-crafted optimizers like SGD or Adam. Even if learned optimizers (L2Os) eventually outpace hand-crafted ones in practice however, they are still not provably convergent and might fail out of distribution. These are the questions addressed here. Currently, learned optimizers frequently outperform generic hand-crafted optimizers (such as gradient descent) at the beginning of learning but they generally plateau after some time while the generic algorithms continue to make progress and often overtake the learned algorithm as Aesop’s tortoise which overtakes the hare. L2Os also still have a difficult time generalizing out of distribution. \cite{heaton_safeguarded_2020} proposed Safeguarded L2O (GL2O) which can take a learned optimizer and safeguard it with a generic learning algorithm so that by conditionally switching between the two, the resulting algorithm is provably convergent. We propose a new class of Safeguarded L2O, called Loss-Guarded L2O (LGL2O), which is both conceptually simpler and computationally less expensive. The guarding mechanism decides solely based on the expected future loss value of both optimizers. Furthermore, we show theoretical proof of LGL2O's convergence guarantee and empirical results comparing to GL2O and other baselines showing that it combines the best of both L2O and SGD and that in practice converges much better than GL2O.
\end{abstract}


\section{Introduction}

An unambiguous trend in machine learning is that different parts of the pipeline are being automatized. Automatized architecture search is now the state of the art for neural networks \citep{efficientnetv2}, learned dynamic learning rates work better than learning rate schedules \citep{MetzAdaptiveLR}, even reinforcement learning algorithms have been learned \citep{evolveRL, metalearncuriosity, Kirsch2020Improving}. There has also been quite some work on learned optimizers (L2O) (cf. section \ref{sec:background}) though so far, due to computational limits, they work only on small datasets and only beat hand-crafted optimizers for the first thousand steps or so (because that is the horizon they are trained on) and do not generalize well. Some, like Richard Sutton\citep{bitterlesson}, argue that as compute becomes cheaper and more accessible, "learned things" will become much better than "handcrafted things". 
However, even if L2O can outperform designed optimizers, they will still have two flaws: they will not be provably convergent, and they might fail totally out of distribution (different dataset, different type of objects being optimized -- what we call optimizees in this paper, different learning horizons, etc.). A guard addresses those shortcomings. 

\begin{figure}
    \centering
    \includegraphics[width=\columnwidth]{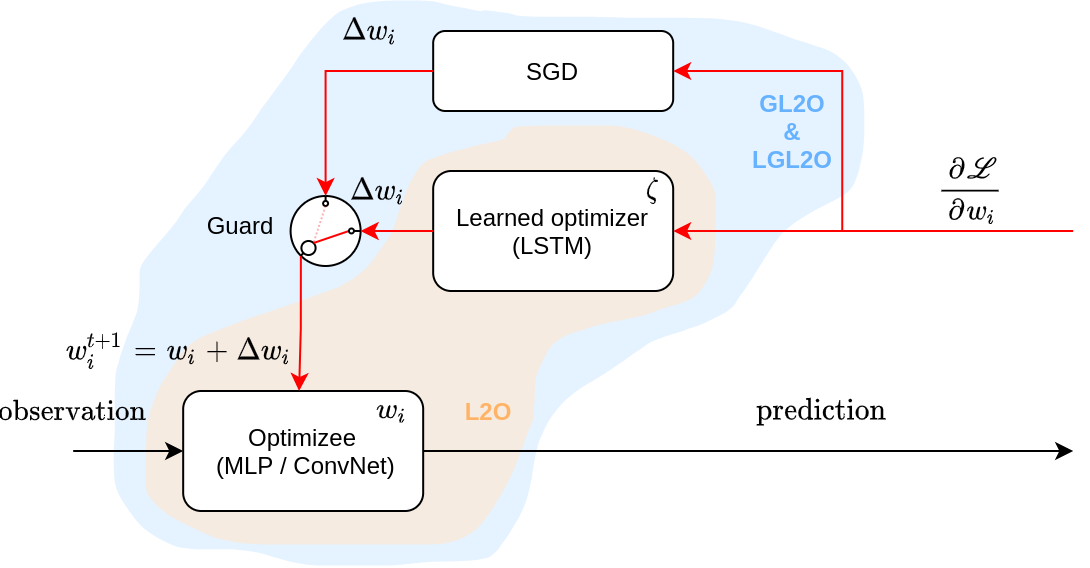}
    \caption{Principle of the guarding mechanism for learned optimizers. A learned optimizer is used to change weights of an optimizee. Guarded learned optimizers (GL2O and LGL2O) add a guard which monitors situations where the optimizer does not perform well and switches to an analytic method (such as SGD) in those cases. This ensures asymptotic convergence of the resulting hybrid optimizer.}
    \label{fig:lgl2o}
\end{figure}

Our main contribution is a guard which takes in as input any blackbox L2O and a provably convergent optimizer and blends them to get the best of both. We prove that our guard keeps the convergence guarantee of the designed optimizer. Then we show in practice that our guard has the desired behaviour, that is, it uses the L2O when the L2O works best (in distribution or in settings where the L2O generalizes well) but correctly switches to the designed optimizer when the L2O underperforms. 

In our experiments, we use \citep{andrychowicz_learning_2016}'s L2O that we train such that it beats other optimizers in the first $\sim$1000 steps, but our contribution, the guard, is independent of the L2O, it can take any L2O as input, and as L2Os get better so will LGL2O. Thus the goal of the paper is not to show that the combination is better than any other existing optimizer, but rather to show that the guard makes the right decisions and that it preserves the advantages of both its input L2O and provably convergent designed optimizer. We show this in our experiments by demonstrating
that LGL2O performs as well as L2O (or better) when L2O outperforms SGD, and as well as SGD (or better) when SGD outperforms L2O. All unguarded L2O approaches currently work only for $\sim$1000 optimization steps in practice when optimizing neural networks. In contrast we show successful optimization up to millions of steps. 

The main reasons we developed a new class of guards, while an existing class of guards \citep{heaton_safeguarded_2020} already exist are:
\begin{itemize}[noitemsep,topsep=0pt,parsep=0pt,partopsep=0pt]
    \item Our guard is conceptually simpler.
    \item Our guard requires fewer hyperparameters.
    \item Our guard requires fewer SGD calls, and those can be done in parallel rather than sequentially.
    \item In practice our guard converges better for neural networks.
\end{itemize}
The first three points are detailed in section \ref{sec:lossguard} while the last point is detailed in section \ref{sec:experiments}.


\section{Related Work}
\label{sec:background}

Learning to Optimize (L2O), popularized by \citet{andrychowicz_learning_2016}, focuses on learning optimization rules using an LSTM network, where gradients, along with other information, are provided on the input to the learned optimizer and updates to the weights of a base network are provided as outputs. Since then, many iterations of similar approaches with various alterations have been proposed \cite{Metz2020,Lv2017,Chen2021}, yet ultimately with the same purpose, i.e. augmenting gradient descent for some practical benefit, such as greater sample efficiency. Historically, however, other works have proposed to use a neural network to train another neural network \cite{prokhorov_adaptive_2002,hochreiter_learning_2001}. For a recent overview of this area, one can refer to \citet{Chen2021}. 

Recently, limitations of L2O approaches have increasingly been analyzed \cite{Metz2020,wu_understanding_2018,metz_overcoming_2021,maheswaranathan_reverse_2020}, revealing a plethora of limitations that hinder the use of L2O methods in practice. Most recently, \citet{heaton_safeguarded_2020} proposed a method for safeguarding the behaviour of any learned optimizer by combining it with stochastic gradient descent (SGD) to confer the hybrid algorithm with convergence guarantees. Our work provides a computationally and conceptually simpler safeguarding mechanism to guarantee convergence. 

Learning to optimize approaches belong, more broadly, to the field of meta-learning and learning to learn in particular \citep{Hospedales2021}. Approaches such as the Badger framework \citep{rosa_badger_2019}, VS-ML \citep{kirsch2020vsml} and BLUR \citep{pmlr-v139-sandler21a} focus on generalizing learning algorithms and architectures and how to discover and train them. The method proposed in this work deals with challenges that must first be overcome to build learning systems that could one day rival and more importantly surpass existing hand engineered solutions, one of the primary motivators for this work. The existing challenges encompass a wide variety of sub-problems. Amongst others, these include a) the ability of the learned algorithm to generalize outside of meta-training distribution and b) the stable long-term (asymptotic) behaviour of the algorithm. These two topics are the subject of this paper.

\section{Loss Guarding}
\label{sec:lossguard}

The principle of guarding is illustrated in Figure \ref{fig:lgl2o}. In L2O, the (neural network-based) learned optimizer receives one gradient per parameter of the optimizee (in what follows we call \textbf{optimizee} whatever gets optimized by the optimizer, in our case it is typically a neural network) and independently proposes a corresponding delta for each parameter that is then applied. In case of guarding (GL2O and LGL2O), the guard looks at parameters proposed by L2O and decides whether to accept them or those suggested by traditional convergence-guaranteed optimizer instead (the fallback update). The guarding mechanism uses the guarding criteria to choose whether to use updates proposed by L2O or the updates proposed by the convergence-guaranteed optimizer (eg. SGD). At every step if the L2O updates are chosen then the L2O weight updates are applied to all the weights, and if the convergence-guaranteed optimizer weight updates are chosen, then those updates are applied to all the weights.

The difference between LGL2O and GL2O is the criterion to decide whether to apply the L2O update or the convergence-guaranteed optimizer update. In the case of GL2O, motivated by the convergence of Cauchy sequences to fixed points in a complete space, the update proposed by L2O is tested by applying normal SGD to it. If the L2O update makes the SGD step size generally smaller (smaller than a Cauchy sequence which tracks the size of accepted steps), then the L2O update is accepted, if not it is rejected and an SGD update is used instead. In contrast, LGL2O, motivated by the idea that convergence in the loss-space implies convergence in the weight-space when the loss function is continuous (and locally convex), simply compares the loss of the L2O update versus the loss of a convergence-guaranteed optimizer like SGD and simply implements the update with the lowest loss. That is at optimization step $k$, if $y_k$ is the point proposed by L2O and $z_k$ is the point proposed by the convergence-guaranteed fallback optimizer, then the criterion to determine the update is 
\begin{align}
    \label{lgl2ocriterion}
    \Ltask(y_k) < \Ltask(z_k) , 
\end{align}
where $\Ltask$ is the loss function.
If the criterion is passed, the algorithm updates to $y_k$ otherwise $z_k$. 

This means that for every weight update of the optimizee, GL2O needs to make one extra call of obtaining the gradients of the optimizee (on the proposal of L2O) and it must be made sequentially, after making the initial one and running L2O. This significantly increases the time complexity of GL2O compared to LGL2O. The logic of GL2O is also thus more complicated.  

\begin{algorithm}[htb]
    \caption{Loss Guarded L2O with (deterministic) gradient descent \label{alg:lossguard}}
    \begin{algorithmic}[1]
        \STATE \textbf{Given} task loss function: $\Ltask$

        \STATE \textbf{Given} L2O operator:  $\lo$

        \STATE \textbf{Given} L2O weights $\{\zeta\}$
                                \hfill {\it $\vartriangleleft$ Take from Meta-Training}

        \STATE \textbf{Given} initial state $x^1 \in \R^n$
                                \hfill {\it $\vartriangleleft$ Initialize iterate}
    \FOR{$k=1,2,\ldots$}

        \STATE $y^k \leftarrow \lo(x^k; \ \zeta)$
                                \hfill $\vartriangleleft$ {\it L2O Update}
        \STATE $z^k \leftarrow x^k - \lambda \nabla \Ltask(x^k)$
                                \hfill $\vartriangleleft$ {\it  Fallback Update}

    \STATE {\bf if} $ \Ltask(y^k) < \Ltask(z^k)$ {\bf then}
                                \hfill {\it $\vartriangleleft$ Safeguard Check}

    \STATE \hspace*{10pt} $x^{k+1} \leftarrow y^k$
                                \hfill {\it $\vartriangleleft$ L2O Update}

    \STATE {\bf else}

    \STATE \hspace*{10pt} $x^{k+1} \leftarrow z^k$
                                \hfill {\it $\vartriangleleft$ Fallback Optimizer Update}

    \STATE {\bf end if}

    \ENDFOR
    \end{algorithmic}
\end{algorithm}

\begin{algorithm}[htb]
    \caption{Loss Guarded L2O with stochastic gradient descent \label{alg:sgdlossguard}}
    \begin{algorithmic}[1]
        \STATE \textbf{Given} mini-batch loss function: $\Ltask$

        \STATE \textbf{Given} L2O operator:  $\lo$

        \STATE \textbf{Given} L2O weights $\{\zeta^k\}$
                                \hfill {\it $\vartriangleleft$ Take from Meta-Training}

        \STATE \textbf{Hyperparameters: } $n_t, n_c, L \in \N^3$

        \STATE \textbf{Given} initial state $x^1 \in \R^n$
                                \hfill {\it $\vartriangleleft$ Initialize iterate}
        \STATE $k \leftarrow 1$
    \WHILE{$k < L$}
        \STATE Sample $n_t$ train mini-batches $\mathcal{B}_t = [b_1, \ldots, b_{n_t}]$
        \STATE Sample $n_c$ validation mini-batches $\mathcal{B}_v = [v_1, \ldots, v_{n_c}]$
        \STATE $y^k \leftarrow x^k$ \hfill $\vartriangleleft$ {\it L2O Init.}
        \STATE $z^k \leftarrow x^k$ \hfill $\vartriangleleft$ {\it SGD Init.}
        \FOR{$i\in [0, \ldots, n_t - 1]$}
            \STATE $y^{k+i+1} \leftarrow \lo(y^{k+i}, b_{i+1}; \ \zeta)$
                                \hfill $\vartriangleleft$ {\it L2O Update}
            \STATE $z^{k+i+1} \leftarrow z^{k+i} - \lambda_{k+i} \nabla \Ltask(z^{k+i}, b_{i+1})$
                                \hfill $\vartriangleleft$ {\it  Fallback Update}
        \ENDFOR
        
    \STATE {\bf if} $ \frac{1}{n_c}\sum_{j=1}^{n_c}\Ltask(y^{k+n_t}, v_j) < \frac{1}{n_c}\sum_{j=1}^{n_c}\Ltask(z^{k+n_t}, v_j) $ {\bf then}
                                \hfill {\it $\vartriangleleft$ Safeguard Check}

    \STATE \hspace*{10pt} $x^{k+n_t} \leftarrow y^{k+n_t}$
                                \hfill {\it $\vartriangleleft$ L2O Update}

    \STATE {\bf else}

    \STATE \hspace*{10pt} $x^{k+n_t} \leftarrow z^{k+n_t}$
                                \hfill {\it $\vartriangleleft$ Fallback Optimizer Update}
    
    \STATE {\bf end if}
    \STATE $k\leftarrow k + n_t$
    \ENDWHILE
    \end{algorithmic}
\end{algorithm}

In algorithm \ref{alg:sgdlossguard}, $n_t$ is number of sequential application of L2O (and in parallel, on the same $n_t$ mini-batches, SGD) before choosing whether to use L2O updates or the fallback SGD updates using criterion \ref{lgl2ocriterion}. $n_c$ is the number of mini-batches (drawn from the training data) used to approximate the loss in criterion \ref{lgl2ocriterion}. In practice we want to choose $n_c \simeq n_t$ for algorithmic speed so that not too many loss function evaluations are needed per optimization step. If we choose $n_c = n_t$, then we need only two loss function evaluation per optimization step while still being able to approximate the total loss function in criterion \ref{lgl2ocriterion} with an arbitrary number $n_c$ of mini-batches. In all our experiments, both $n_t$ and $n_c$ are chosen to be 10. These are the only two hyperparameters of the algorithm \ref{alg:sgdlossguard} (except for $L$, which is how long one chooses to optimize for). This compare favourably with GL2O which has the choice of 5 possible sequence types, each choice then with typically 2 hyperparameters to tune.

\begin{thm} \label{thm: lossguard-convergence}
Let $\Ltask$ be a continuous loss function which is $\mu$-strongly convex\footnote{We remark that any convex function can be turned into a strongly convex one simply by adding an $L_2$ regularization. In the non-convex case, we get convergence to a local minimum instead of the global one. } and $L$-smooth and let $w^*$ be it's global minimum. Let $w_{i\in\N}$ be a sequence of points obtained from applying the Loss-Guarded L2O algorithm with gradient descent or stochastic gradient as the guarding algorithm. In the case of stochastic gradient descent, we assume that in expectation, the stochastic gradient $\nabla_{mb} \Ltask(w)$ is equal to the true gradient, 
$$\Esp(\nabla_{mb} \Ltask(w)) = \nabla \Ltask(w) , $$ and that the variance of the stochastic gradient around the true gradient is bounded. Then given a constant learning rate $0<\lambda<\min(\frac{2}{L},2\mu)$ for gradient descent or a decaying learning rate $\lambda_i\propto\frac{1}{i_0+i}$ for SGD, the sequence converges to the minimum, i.e. $$\substack{\lim \\ i\rightarrow \infty} w_i = w^* .$$ 
\end{thm}
\begin{proof}
This is simply a combination of Theorems \ref{thm:conv} and \ref{thm:sgdconv} proven in Appendix.
\end{proof}

Because the LGL2O criterion depends on comparing the proposed points coming from the L2O and SGD on the whole loss function, but in all our experiments the loss function is the sum of the individual losses on many dataset points which would take considerable compute to evaluate, in practice we approximated the loss function in the criterion with 10 mini-batches of data. There is a risk in making this approximation: if the error on the loss from approximating it with a limited number of mini-batches becomes similar to the difference in the loss values of the SGD and L2O update proposals, then there is a risk that the algorithm will choose the incorrect update and convergence will not be guaranteed. This is something which risks happening near a local optimum, in which case one should increase the number of mini-batches used in the evaluation of the loss criterion to lower the approximation error. 

\section{Experiments}
\label{sec:experiments}

\begin{figure}[htb]
    \includegraphics[width=\columnwidth]{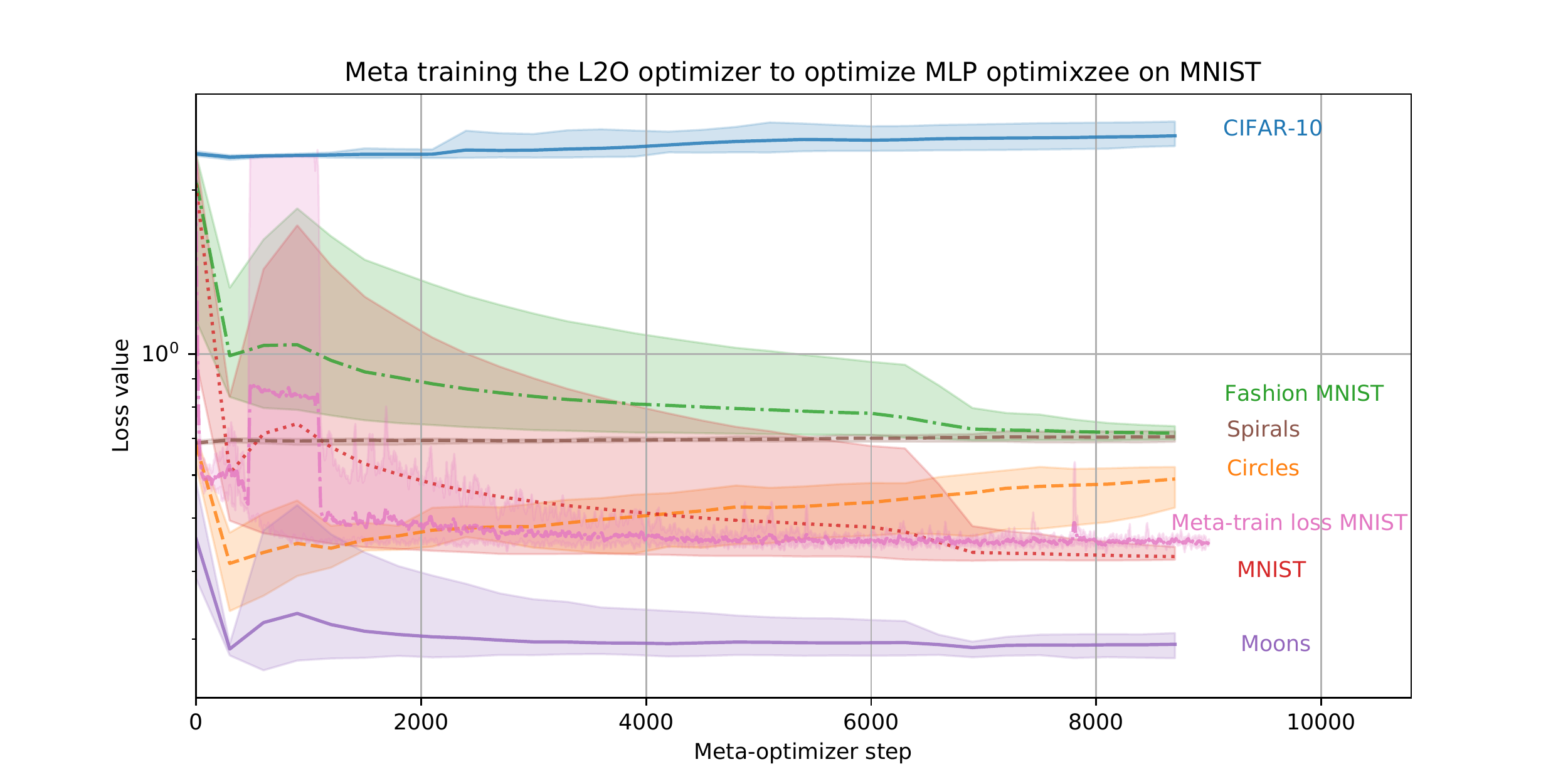}
    \caption{Outer-loop convergence of the learned optimizer on the MNIST dataset. The L2O method achieves reasonable performance on the MNIST dataset. It generalizes to FashionMNIST \& Moons \& Circles datasets, while the MLP seems too small to have sufficient capacity to solve the Spirals \& CIFAR-10. Mean and min-max ranges are across 5 runs.}
    \label{fig:meta-learning}
\end{figure}

This chapter compares the proposed LGL2O with the original GL2O, non-guarded L2O and baseline handcrafted algorithms. We first compare them in distribution, that is when L2O is used like it was meta-trained, on the same dataset, and with the same optimizee. In all our experiments, we use an L2O that was meta-trained on MNIST on a small fully-connected MLP for roughly 10k meta-optimizer steps. Then we follow with out of distribution experiments. First with only the dataset being out of distribution (other than MNIST), then with only the optimizee being out of distribution (ConvNets instead of MLPs), and finally both.

To show that the success of our guard is not dependent on the above specific L2O, we also meta-trained another L2O in a different context (Convnet on CIFAR-10) and show in appendix \ref{sec:different} that our guard works just as well with this other learned optimizer as with the MLP-MNIST-metatrained L2O we used in this section.

The experiments were conducted on publicly available datasets, namely MNIST \citep{mnist}, FashionMNIST \citep{fashion}, CIFAR10 \citep{cifar10}, TinyImagenet\footnote{TinyImagenet dataset is publicly available from Kaggle competition website at \url{https://www.kaggle.com/c/tiny-imagenet/data}} - a subset of Imagenet dataset \cite{imagenet} and simple datasets from the Scikit-learn library \citep{scikit-learn}.

The purpose of these experiments is to illustrate a common property of meta-learned algorithms: that they tend to converge faster than analytic algorithms and then they plateau. It is expected that the guarding mechanism in GL2O will eventually take over and switch to the SGD-based guard, thus assuring the asymptotic convergence towards the optimal solution.

The learned optimizer (and it's weights) is identical in all experiments and consists of an LSTM \citep{lstm} with 2 hidden layers of 20 cells each and a linear output layer which was meta-trained with a rollout-length of 100 steps to optimize an MLP on the MNIST dataset. Pre-processing of the input is used as described in \citep{andrychowicz_learning_2016}.

First, the optimizer was meta-trained to optimize the optimizee on the MNIST dataset as shown in the Figure \ref{fig:meta-learning}. The optimizee used for meta-training is an MLP with 1 hidden layer with 20 neurons. It uses sigmoid activations in hidden layer(s), softmax on the output layer and a negative log-likelihood loss. 

\begin{table}[h!]
\begin{tabular}{llll}
\toprule
\multicolumn{2}{l}{In Distribution?} & \multicolumn{2}{l}{Experiment Type} \\ \cmidrule(lr){1-2}\cmidrule(lr){3-4}
\textbf{Dataset}          & \textbf{Optimizee}         & \textbf{Dataset}               & \textbf{Optimizee}   \\ \toprule
yes              & yes               & MNIST                 & MLP         \\
yes              & no                & MNIST                 & Conv        \\
no               & yes               & Spirals \& Circles    & MLP         \\
no               & no                & CIFAR-10              & Conv        \\ 
no               & no                & TinyImagenet50        & Conv         \\ \bottomrule
\end{tabular}
\caption{Overview of the experiments: after the meta-training phase, the ability of the hybrid optimizer LGL2O is evaluated on long rollouts, on in-distribution and out-of-distribution datasets and optimizees.}
\label{table:exp}
\end{table}

After the meta-learning phase, the meta-testing phase begins. At that point, the learned optimizer works in an inference regime. In all these experiments we explore the behaviour of LGL2O ("LGL2O (ours)") in the plots) is compared with Guarded L2O from \citet{heaton_safeguarded_2020} (GL2O in the plots), vanilla L2O (L2O in the plots) and SGD without momentum (SGDnm) (which the fallback optimizer which we use inside LGL2O). In addition, each figure plots "LGL2O use\_l2o" which tracks whether the L2O update (use\_l2o=1) or the SGD update (use\_l2o=0.5) was chosen by the loss-guarding criterion (\ref{lgl2ocriterion}) on this optimization step. 

The ability to generalize to longer rollouts and to out-of-distribution dataset and network architectures is systematically evaluated on the following experiments, the experiments are described in table \ref{table:exp}. In all graphs, the data are collected from 5 independent runs with different seeds. The dark lines trace the mean of the 5 seeds and the shading the minimum and maximum.

In practice (in all experiments shown here) the following schedule for guarding in LGL2O is used:
\begin{itemize}
\item make 10 optimizer steps on 10 consecutive mini-batches (for both L2O and guard),
\item compute loss of the resulting optimizee as an average of 10 unseen mini-batches (for both optimizees produced by L2O and guard).
\end{itemize}

This scheme has two motivations. First: a typical learned optimizer behaves differently than SGD, it usually starts with a form of triangulation of the loss landscape, during which the loss increases considerably (see the initial steps in Figure \ref{fig:spirals_mlp}) and then the loss starts decreasing. We observed that this stage is necessary for L2O to work. This is the reason why each of the optimizers do 10 consecutive steps first. The second part (evaluating loss over 10 "test" mini-batches) aims to inhibit influence of noise in the evaluation. Without averaging, LGL2O exhibited very unstable behaviour caused by switching from the guard back and forth inappropriately.  

\subsection{In-distribution dataset and optimizee}
\label{sec:indistribution}

\begin{figure}[htb]
    \centering
    \includegraphics[width=\columnwidth]{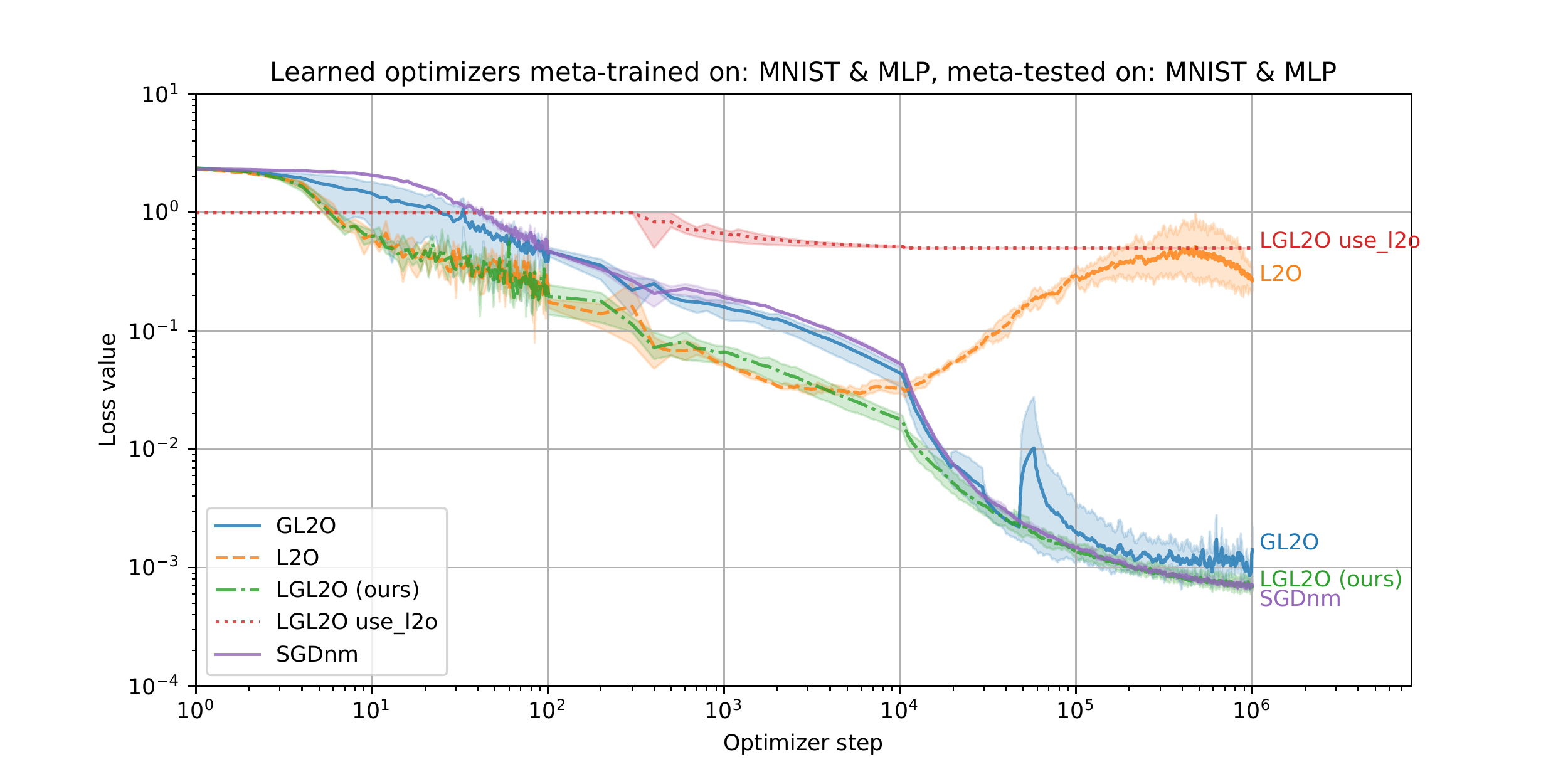}
    \caption{\textbf{In-distribution dataset and optimizee}. A typical behaviour of L2O is shown here: at beginning, it converges faster than other algorithms, however, after around 1000 steps it starts diverging. Compared to this, the LGL2O algorithm detects that L2O performs worse than the fallback optimizer, so it starts using the fallback optimizer (the `use\_l2o` line goes towards value 0.5, which corresponds to only using the fallback) and thus stays as good or better than either components (L2O and SGDnm) at all times. As a result, LGL2O combines benefits of both worlds: thanks to the L2O optimizer, it converges quickly at the beginning, while it preserves asymptotic convergence of the fallback optimizer (SGD without momentum here). Note: on all graphs a moving window averaging is used, which starts at Optimizer step (x-axis) 300. The following is common for all the figures below: \textbf{LGL2O:} loss of our LGL2O algorithm, \textbf{GL2O:} loss of the GL2O algorithm of \citep{heaton_safeguarded_2020}, \textbf{L2O:} loss of vanilla L2O, \textbf{SGDnm:} loss of SGD without momentum, \textbf{LGL2O use\_l2o:} indicator function which indicates whether on this step LGL2O used the L2O update (=1) or the SGD update (=0.5)}
    \label{fig:mnist_mlp}
\end{figure}

This experiment uses the dataset and optimizee that was used during the meta-training phase. Therefore it evaluates the asymptotic behaviour of the learned L2O optimizer. It illustrates the typical flaw of learned optimizers well, which is to say that they perform extremely well at the beggining of optimization but then start diverging.

Figure \ref{fig:mnist_mlp} shows that L2O converges quickly at first, but eventually the loss starts diverging. The LGL2O algorithm solves this by detecting when the handcrafted optimizer would be preferable and switching to it. It can be seen that the switching is not definitive, once L2O performs better than the guard, it can be used again instead of the guard. As expected, the resulting hybrid algorithm steadily converges towards the optimum.

The motivation of guarded learned optimizers is to combine the best of both worlds: to achieve quick convergence at the beginning and then be able to keep converging towards the optimum. LGL2O shows this kind of behaviour in Figure~\ref{fig:mnist_mlp}.

We also observe that while GL2O prevents the divergence of the L2O, it does at the cost of much of the performance gains of L2O at the beginning of optimization (between steps 0 and 100, where GL2O is better than SGDnm but much worse than vanilla L2O).

\subsection{Out-of distribution dataset, in-distribution optimizee}
\label{sec:outdistribution}

\begin{figure}[htb]
    \centering
    \includegraphics[width=\columnwidth]{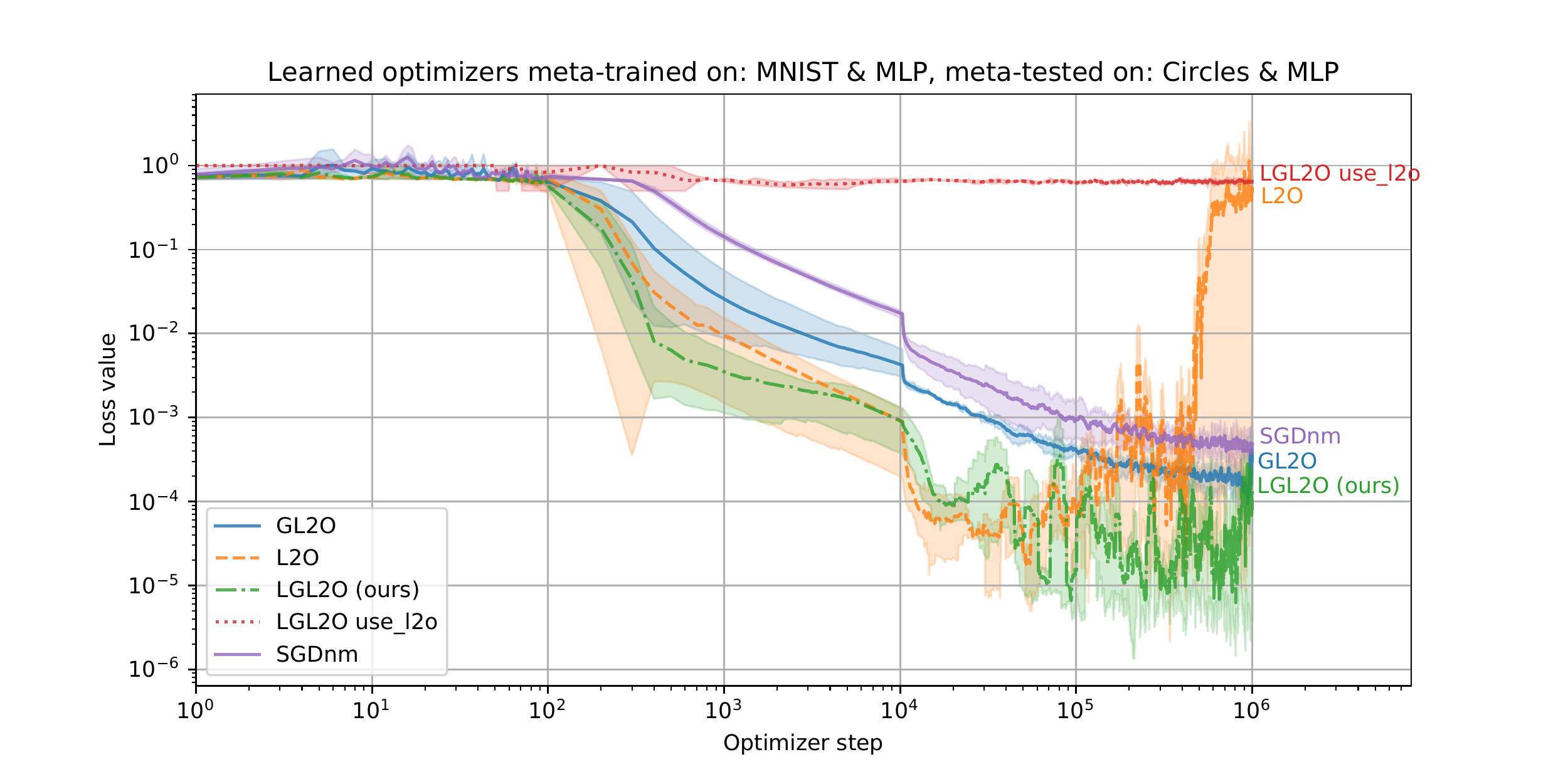}
\caption{\textbf{Out-of-distribution dataset, in-distribution optimizee.} Performance of LGL2O is comparable to L2O at the beginning, where all the optimizers converge quickly. However, L2O starts diverging later in the rollout, while LGL2O correctly switches towards the guard. Again we see that GL2O prevents divergence but at the cost of underperforming compared to L2O at the beginning. We suspect\footnotemark that the reason LGL2O starts diverging after 20k steps is because we evaluate the criterion on only 10 mini-batches, and when the loss is already so low, the variance between different groups of 10 mini-batches becomes larger than the difference between the losses of the SGD proposed update and L2O proposed update, so our implementation of the algorithm starts incorrectly using the divergent L2O too often.}
    \label{fig:circles_mlp}
\end{figure}
\footnotetext{In fact, we have since run experiments, shown in Appendix \ref{sec:stability}, showing that in accordance with our suspicions, increasing the values of the hyperparameters $n_t$ and $n_c$ indeed solves the instability.}

\begin{figure}[htb]
    \centering
    \includegraphics[width=\columnwidth]{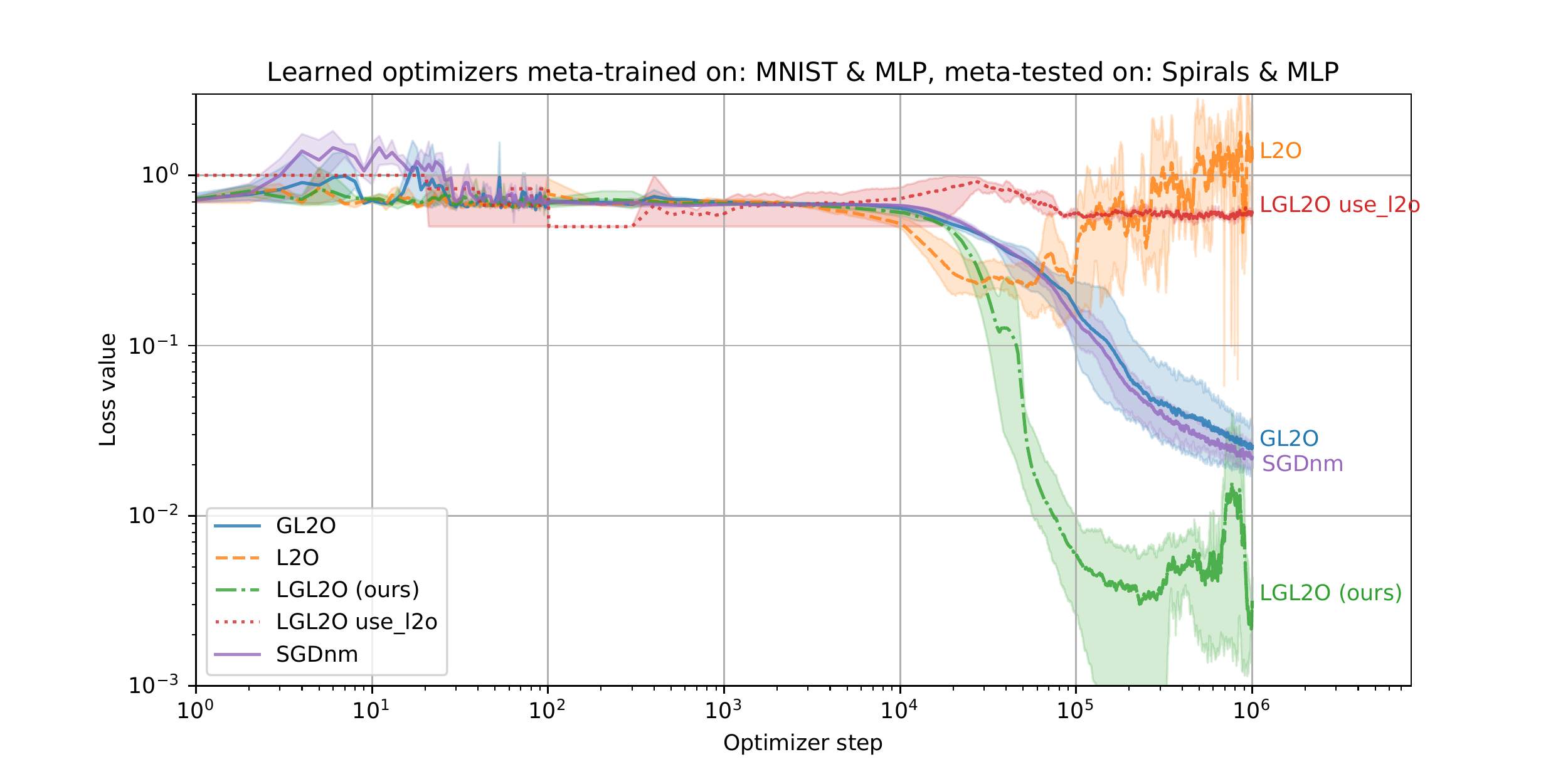}
    \caption{\textbf{Out-of-distribution dataset, in-distribution optimizee.} MLP optimizee optimized using the Spirals dataset. The dataset is relatively challenging for MLP with 1 small hidden layer. It can be seen that the SGD baseline without momentum gets stuck at a suboptimal solution forever and Adam gets stuck there for a very long time. While all L2O-based optimizers go past this optimum, it is noticeable that the guarding mechanism in the original GL2O algorithm slows down its convergence. Compared to that, LGL2O converges as fast as L2O while maintaining reasonable asymptotic stability. And more impressively, after 50k steps, LGL2O significantly outperforms both it's constituent parts (L2O and SGDnm) showing that a judicious switching between the two provides synergistic gains.}
    \label{fig:spirals_mlp}
\end{figure}

This experiment evaluates performance of the L2O optimizer on optimizing an MLP on out-of-distribution, simple 2D datasets: Circles \& Spirals.

On the Circles dataset, LGL2O behaves as expected and keeps converging in later stages of the rollout. Here, LGL2O is more noisy in the later stages than GL2O. One hypothesis is that this could be caused by noise in the optimizer fitness evaluation and could be fixed by averaging loss values over more mini-batches\footnote{Both optimizers do $n_t=10$ consecutive steps, then their loss is computed as a mean over $n_c=10$ new mini-batches.}. In fact, this suggest a future improvement we could make to the LGL2O algorithm \ref{alg:sgdlossguard}, to adaptively increase $n_t$ and $n_c$ (keeping the ratio $n_t:n_c$ constant not to increase the number of function evaluations per step) when the mean plus or minus the standard deviation of $\left\{\Ltask(y^{k+n_t}, v_j)\right\}_{k\in[1,n_c]}$ and $\left\{\Ltask(z^{k+n_t}, v_j)\right\}_{k\in[1,n_c]}$ start to overlap in line 16 of algorithm \ref{alg:sgdlossguard}. Indeed, in supplementary experiments in Appendix \ref{sec:stability}, we show that having larger $n_t$ and $n_c$ fixes the late stage noisiness of LGL2O on the Circles dataset.
 
The experiment on Spirals (see Figure \ref{fig:spirals_mlp}) illustrates the incredible synergy operated by LGL2O. By judiciously alternating between its constituent parts (L2O and SGDnm), LGL2O is able to significantly outperform both does constituent parts after 50k optimization steps. 
\subsection{In-distribution dataset, out-of-distribution optimizee}
\label{sec:mnist_conv}

\begin{figure}[htb]
    \centering
    \includegraphics[width=\columnwidth]{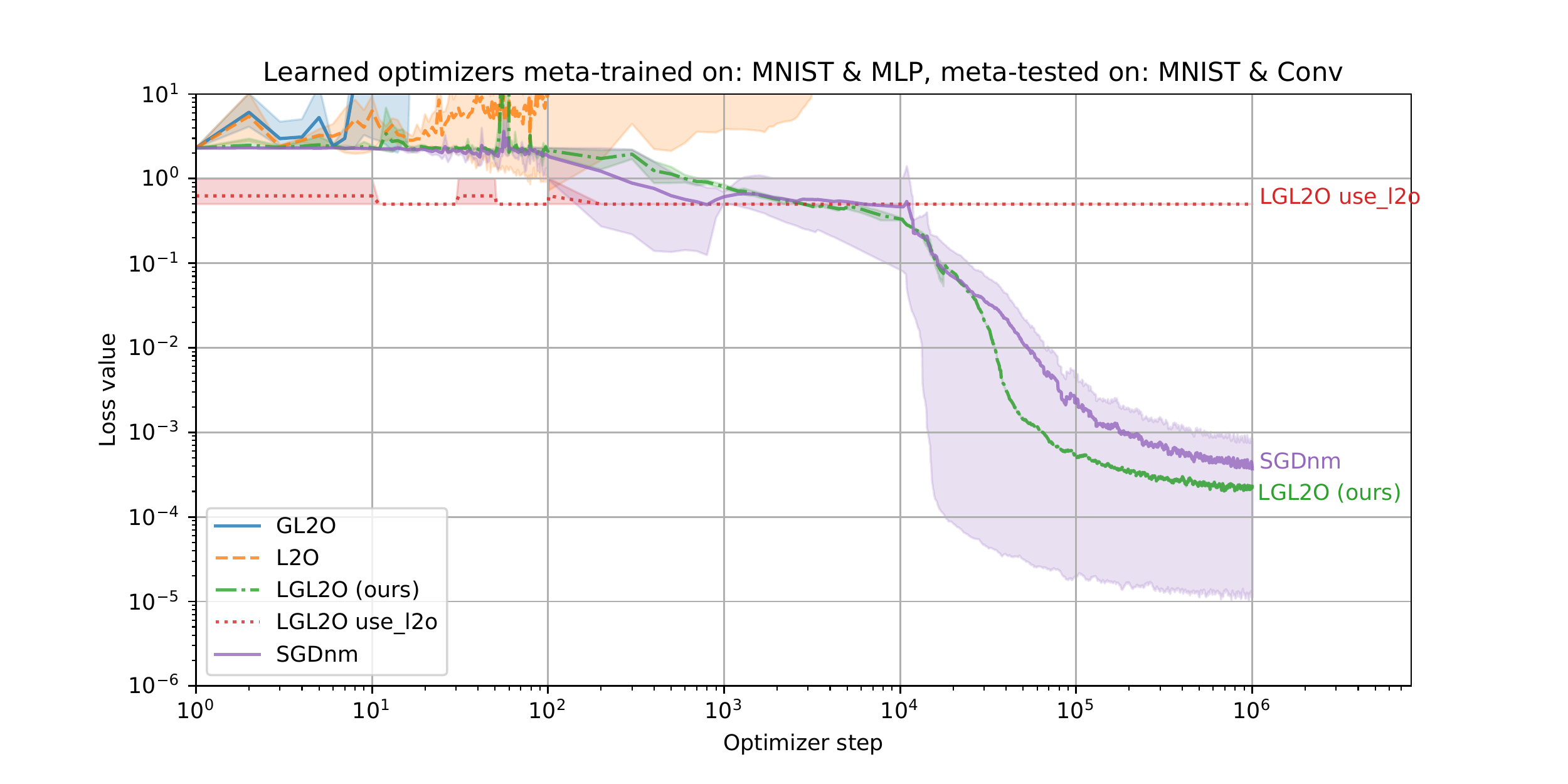}
    \caption{\textbf{In-distribution dataset, out-of-distribution optimizee.} L2O was meta-trained to optimize an MLP, therefore it does not generalize to optimization of CNNs here. LGL2O detects bad behaviour of the learned optimizer and switches to the fallback optimizer, which then performs similarly to the baseline, as one would hope.}
    \label{fig:mnist_conv}
\end{figure}

This experiment illustrates performance of L2O and LGL2O on MNIST and a convolution neural network (CNN) as an optimizee. The optimizee is a 3-layer CNN with $number\_of\_channels=(8,16,32)$, $kernel\_sizes=(5,3,3)$ and $strides=(2,2,2)$, with a fully connected final layer.

Since the CNN optimizee is sufficiently different from the MLP that L2O was meta-trained on, L2O diverges very quickly, as seen in  Figure \ref{fig:mnist_conv}. GL2O diverged as well\footnote{This might be fixed by further tuning of the guarding algorithm hyperparameters but our grid search did not find good hyperparameters for GL2O in this case.}, while LGL2O engages the guard soon enough and manages to converge steadily tracking it's fallback optimizer (SGDnm) as one would hope in the case where the learned optimizer is not helpful.

It is worth noting that while GL2O is guaranteed to converge asymptotically, practical performance in this case very bad. LGL2O however does the best one could expect in a bad situation.

\subsection{Out-of-distribution dataset, out-of-distribution optimizee}
\label{sec:cifar_conv}

This setup evaluates L2O and LGL2O on an out-of-distribution dataset (CIFAR-10) and an out-of-distribution optimizee (ConvNet). Results can be seen in Figure \ref{fig:cifar_conv}.

\begin{figure}[htb]
    \centering
    \includegraphics[width=\columnwidth]{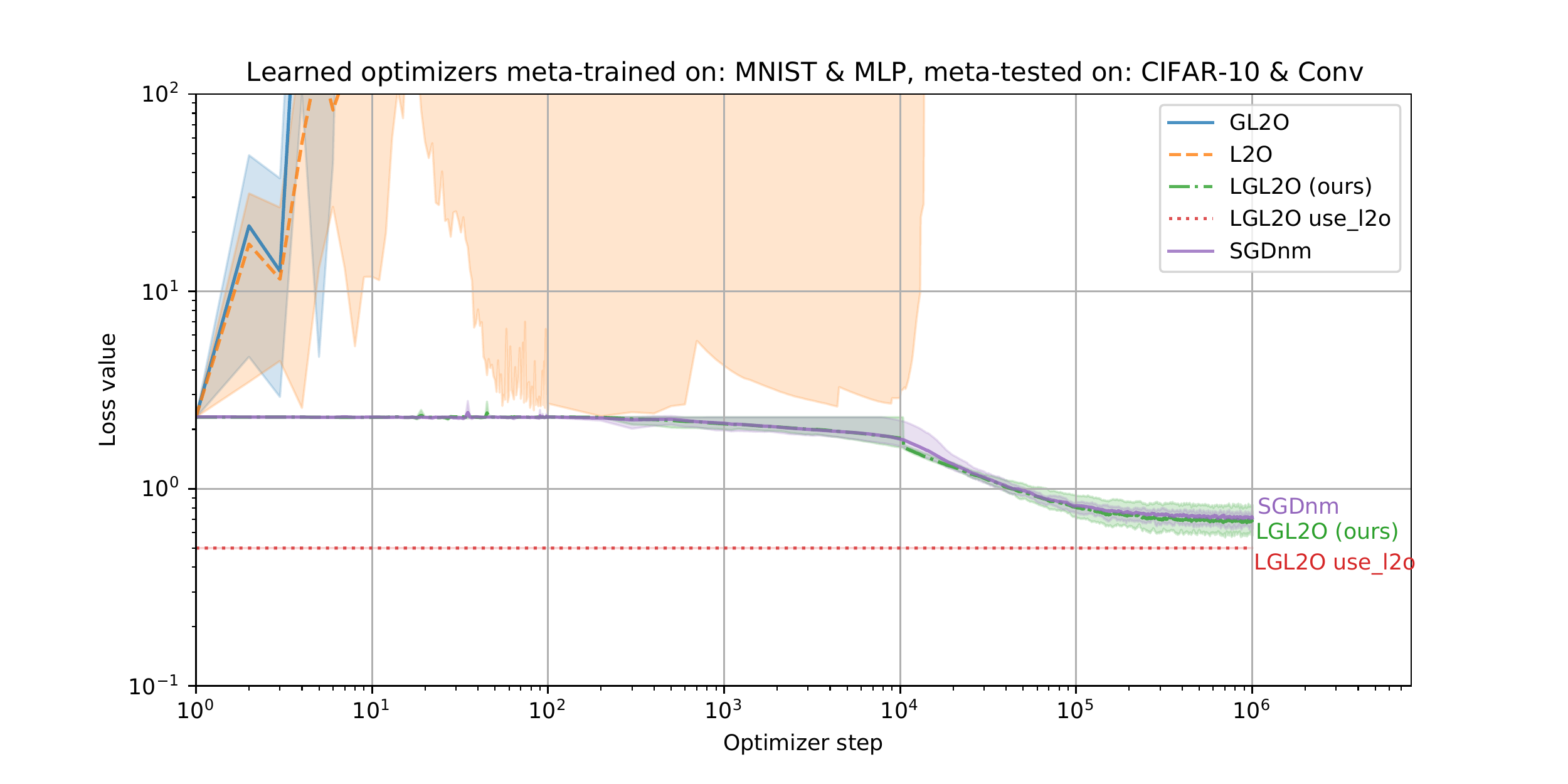}
    \caption{\textbf{Out-of-distribution dataset and out-of-distribution optimizee.} L2O diverges very quickly (due to being out of distribution for both the dataset and optimizee), while LGL2O converges using the fallback optimizer from the beginning.}
    \label{fig:cifar_conv}
\end{figure}

In this case, again, L2O and GL2O completely fail (diverge very quickly to high loss values). While LGL2O uses its fallback optimizer from the beginning and saves the stable convergence.

\subsection{Out-of-distribution dataset (Imagenet), out-of-distribution optimizee}
\label{sec:imagenet_conv}

The last type of setup evaluates L2O, GL2O and our LGL2O on an out-of-distribution dataset (TinyImagenet) and an out-of-distribution optimizee (ConvNet). For the results, a subset of the dataset, which contains randomly chosen 50 (out of full 200) classes was used. Considering the nature of the dataset, a deeper convolutional network was used this time, namely: the optimizee is a 5-layer CNN with $number\_of\_channels=(8,16,32,32,64)$, $kernel\_sizes=(3,3,3,3,1)$ and $strides=(2,2,2,1,1)$, with a fully connected final layer. In order to achieve a stable convergence of the baseline (and the fallback optimizer at the same time), a grid-search for algorithm hyper-parameters was concluded (see the Fig.\ref{fig:tinyimagenet_conv_grid} for reference). The resulting hyper-parameters used for the TinyImagenet experiment are: $learning\_rate=0.01$ and  $learning\_rate\_decay=50000$, these are used for SGDnm baseline and fallback optimizers for LGL2O and GL2O algorithms. 

Results can be seen in Figure \ref{fig:tinyimagenet_conv}, the L2O diverges towards very high loss values, the GL2O diverges immediately to Not-a-Number loss values (therefore not shown in the graph), while our LGL2O correctly detects the tendency to L2O to diverge and switches to the fallback optimizer. As a results, it behaves almost identically to the SGDnm baseline, as desired.  

\begin{figure}[htb]
    \centering
    \includegraphics[width=\columnwidth]{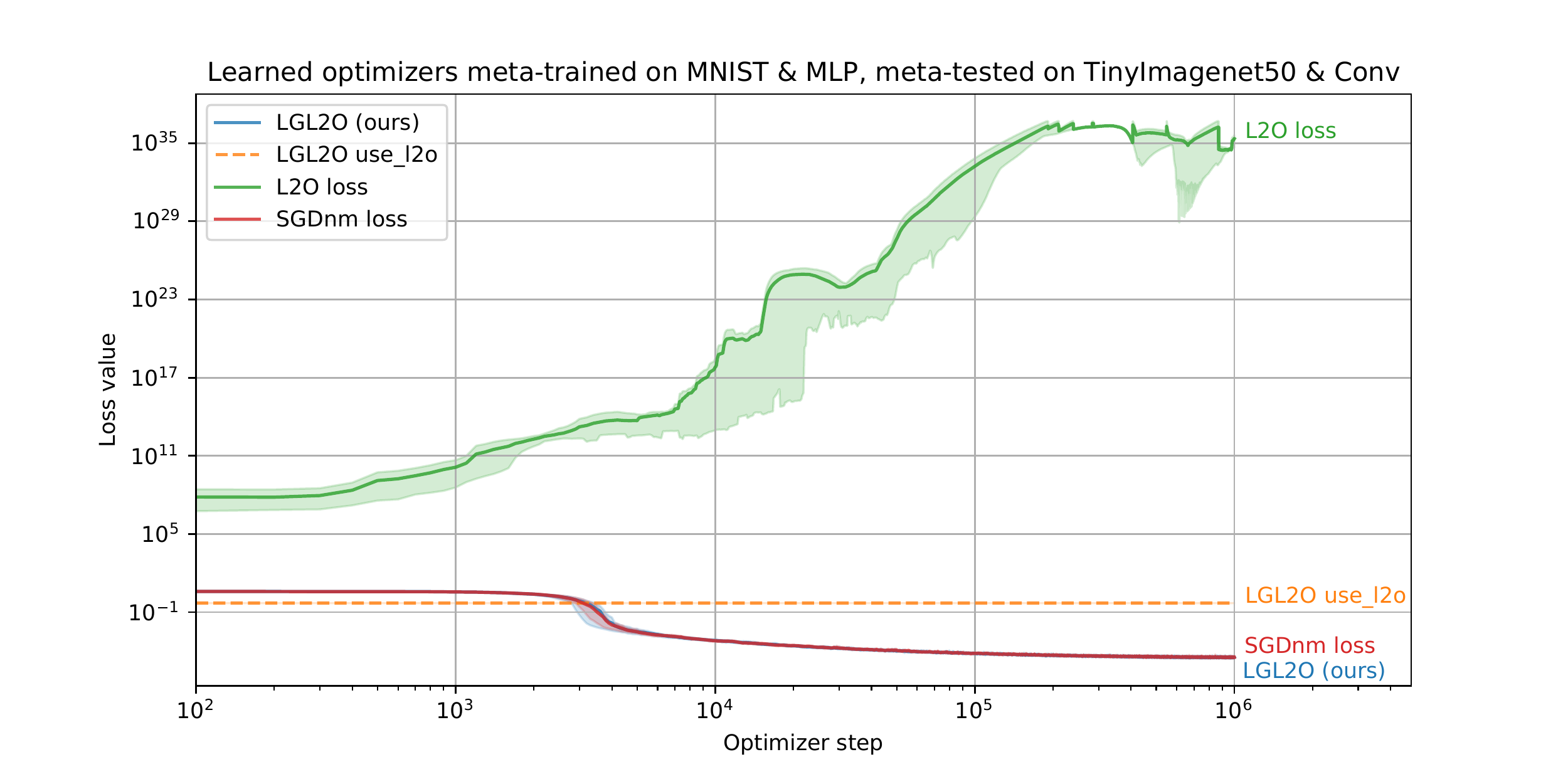}
    \caption{\textbf{Out-of-distribution dataset and out-of-distribution optimizee.} Here, the optimizee (network) has different nature (MLP vs ConvNet) but it has also different size than in the previous experiment. The GL2O algorithm diverges quickly to NaN loss values (ommitted in the graph), L2O diverges to high loss values, while ours, LGL2O, discovers the divergent tendency of the learned L2O and switches immediately to the fallback optimizer. Therefore the L2L2O behaves almost identically to the SGDnm baseline.}
    \label{fig:tinyimagenet_conv}
\end{figure}

\subsection{Faster convergence - Adam as the fallback optimizer of LGL2O}
\label{sec:adam_guard}

\begin{figure}[htb]
    \centering
    \includegraphics[width=\columnwidth]{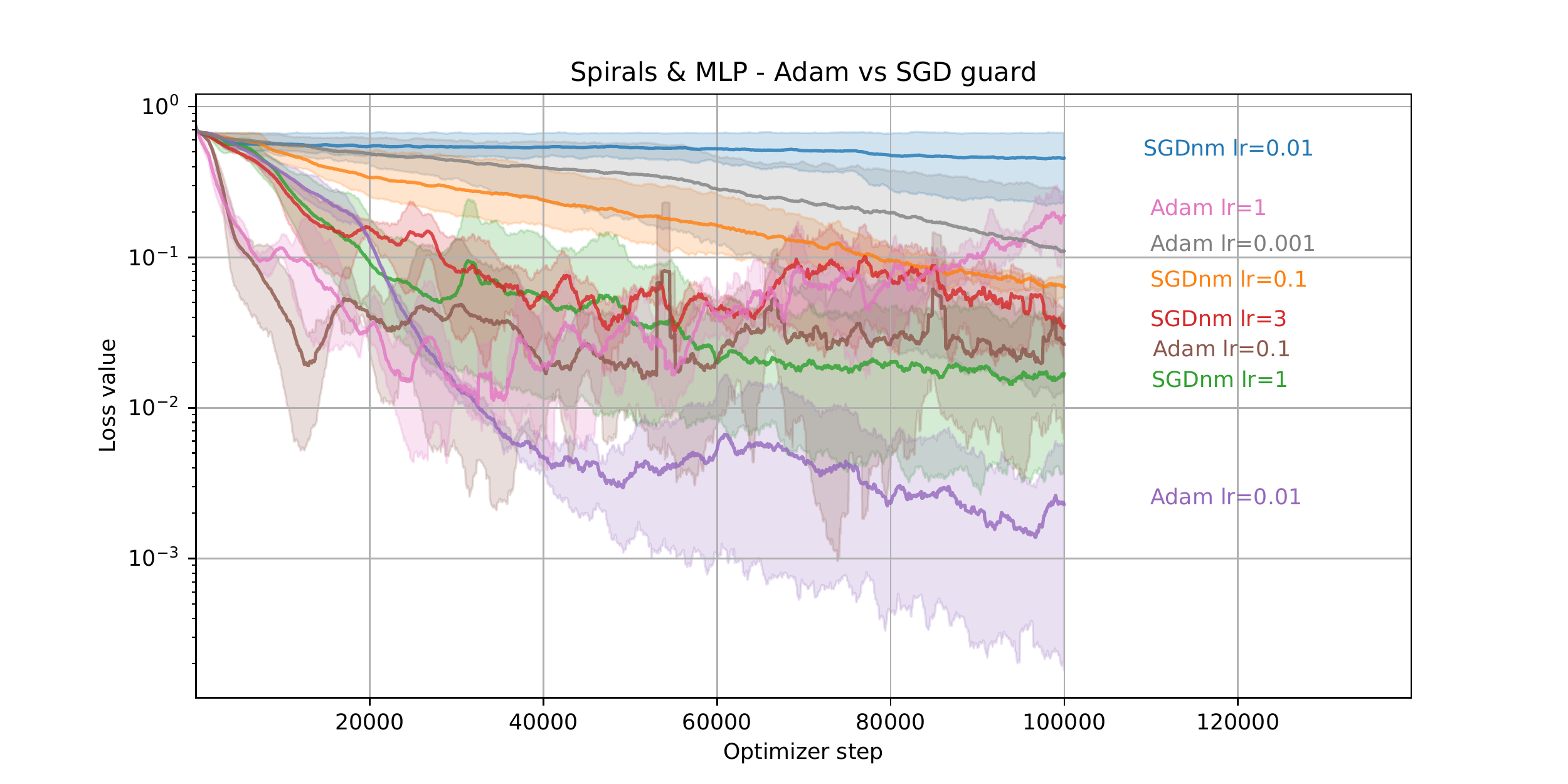}
    \caption{\textbf{LGL2O with a variety of different fallback optimizer.} In this experiment with an out-of-distribution dataset and in-distribution optimizee, we see that we can improve the practical performance of LGL2O by using Adam as a fallback optimizer instead of SGD. The legend designate which fallback optimizer was used in LGL2O as well as its learning rate (lr). \textbf{SGDnm: SGD without moment}. \textbf{SGD: SGD with moment}. \textbf{Adam: Adam}}
    \label{fig:adam_guard}
\end{figure}

In many results shown above, the Adam baseline performed better than LGL2O in the long run (though for sample-efficiency in the beginning LGL2O outperformed Adam in cases where L2O generalized). This is caused by the fact that LGL2O was using SGD (without momentum) as a fallback optimizer. This experiment shows that by using the Adam optimizer as a fallback optimizer, it is possible to achieve considerably better results than with SGD (without momentum), see Figure \ref{fig:adam_guard}. The graph shows convergence of LGL2O optimizing an MLP optimizee on the Spirals dataset. 

This paper focused on using an SGD fallback optimizer, the fallback optimizer needs to itself have a convergence guarantee in order to transmit that property to the LGL2O.

\section{Conclusion}
\label{sec:conclusion}

Most types of learned optimizers have common problems, they either fail to generalize outside of the training distribution and/or do not have good asymptotic properties (e.g. they are unable to overfit to a small dataset or get stuck on a suboptimal solution).

It has been shown that these drawbacks can be addressed using an explicit mechanism (i.e. guarding) that combines the good initial behaviour of learned algorithms with the desired asymptotic properties of analytic algorithms (like SGD). 

This work builds on top of GL2O, which hybridizes learned and hand-crafted learning algorithms and combines best of both worlds. The LGL2O algorithm proposed in this paper is conceptually simpler (e.g. the guarding mechanism has fewer hyperparameters), and computationally less expensive, while maintaining convergence guarantees. We also show that it performs better in practice.

It was shown that LGL2O behaves as ideally desired, it converges quickly using the learned L2O initially and then converges steadily in later stages by relying on the hand-crafted guarding algorithm. At all times and in all cases, LGL2O performed as well or better than it's constituent parts (the learned optimizer, and the fallback optimizer SGD without momentum). And at times it was considerably better than both. The same cannot be said of GL2O which while also provably convergent asymptotically, suffered significant performance loss compared to L2O at the beginning of training in the cases where the L2O was in-distribution or generalized.

In addition to having a simpler decision rule than the GL2O algorithm, LGL2O seems to work more robustly, especially in the ConvNet experiments.

Due to its simplicity, the algorithm did not need much explicit hyperparameter tuning. 

The contributions of the paper are the following:
We propose a new type of guard, LGL2O which is conceptually simpler than the exist class of guards (GL2O), has smaller time complexity, has fewer hyperparameters, and converges better in practice. We prove convergence guarantees for LGL2O. And finally we show that in practice it performs better than GL2O and vanilla L2O. In particular, we show LGL2O can allow the use of learned optimizers without divergence for millions of optimization steps which up until now could never exceed thousands of steps.

In general, the topic of hybridization of the learned algorithm was shown on the case of simple L2O, but the mechanisms shown here are completely general and it should be beneficial to apply them to other, more complex, meta-learning architectures. Further investigation in these directions is left for future work. 

\section*{Acknowledgements}
We would like to thank Martin Poliak,  Nicholas Guttenberg, and David Castillo for their very helpful comments and discussions. Last, but not least, we would like to thank to Marek Rosa, who enabled this research to happen in the first place. 


\clearpage 
\clearpage 
\bibliography{references}  
\bibliographystyle{icml2022}


\clearpage
\appendix

\newpage
\section{Appendix}
This appendix contains the proofs of convergence guarantees, the hyperparameters used and how they were obtained, how randomness was used in the algorithms, and which hardware and software were used as well as addition experiments of the LGL2O Guard with a different learned optimizer to show the generality of our guard.

\textbf{Table of Contents:}\newline
\ref{sec:additionalfigs} Additional Figures \newline
\ref{sec:proof} Proof of Convergence Guarantee\newline
\ref{sec:different} Experiments with Different Learned Optimizer\newline
\ref{sec:hyperparams} Hyperparameters\newline
\ref{sec:randomness} Randomness\newline
\ref{sec:hardandsofware} Hardware and Software

\section{Additional Figures}
\label{sec:additionalfigs}

\begin{figure}[htb]
    \centering
    \includegraphics[width=\columnwidth]{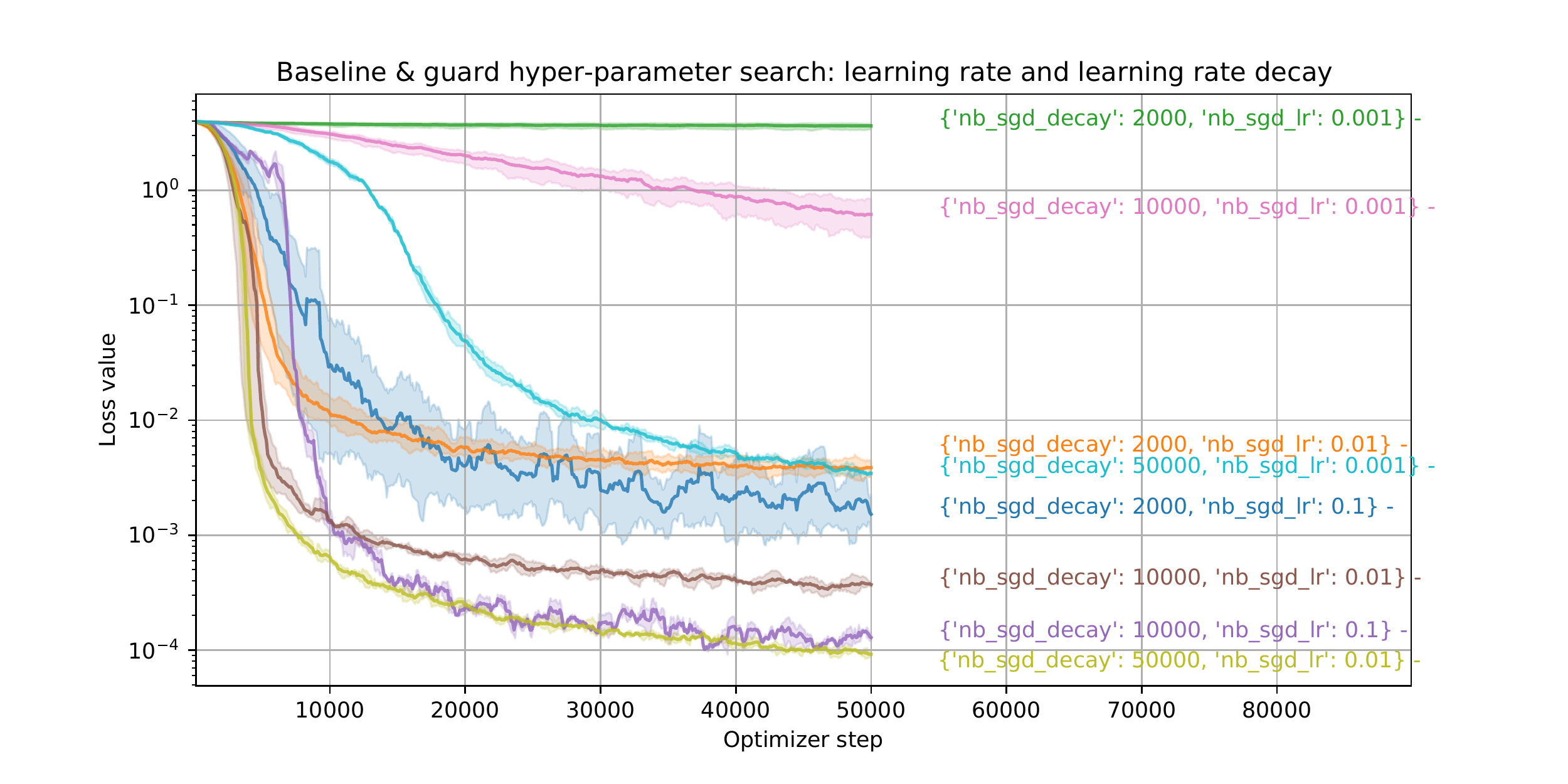}
    \caption{Since the ConvNet optimizee has different topology and size than other methods, additional grid-search for the baseline and the guard hyper-parameters (SGDnm) was made. Results are concluded over 5 runs. The resulting hyperparameters used for the TinyImagenet experiment are: $learning\_rate=0.01$ and  $learning\_rate\_decay=50000$.}
    \label{fig:tinyimagenet_conv_grid}
\end{figure}

\section{Proof of Convergence Guarantee}
\label{sec:proof}

Many of the proofs in this section have been strongly inspired by \citep{marianablog} and \citep{gowerproofs}

\begin{prop}
\label{prop:lgl2olsmooth}
Let $f$ be a continuous loss function which is $L$-smooth and let $w^*$ be its global minimum. Let $E(w) := f(w) - f(w^*)$. Let $w_{i+1}$ and $w_{i}$ be two subsequent points in a sequence generated by the Loss-Guarded L2O algorithm using (deterministic) gradient descent with a step size of $\alpha/L$ for the guarding mechanism. Then:
\begin{align}
\label{lsmootheq1}
    E(w_{i+1}) - E(w_i) \leq -\frac{\alpha}{L}(1 -\frac{\alpha}{2}) \lVert \nabla f(w_i)\rVert_2^2 
\end{align}
and 
\begin{align}
\label{lsmootheq2}
    E(w_i) \geq \frac{\alpha}{L}(1 -\frac{\alpha}{2}) \lVert \nabla f(w_i)\rVert_2^2 
\end{align}
\end{prop}
\begin{proof}
If $w_{i+1}$ is chosen via the guarding mechanism then we have that $w_{i+1} = w_i - \frac{\alpha}{L}\nabla f(w_i)$. Furthermore by virtue of $f$ being $L$-smooth we have that 
\begin{align}
    \label{lsmoothprop}
    f(y) \leq f(x) + \langle \nabla f(x), y - x\rangle + \frac{L}{2}\lVert y - x\rVert_2^2 .
\end{align}
Thus assuming that $w_{i+1}$ is chosen via the guarding mechanism we have that
\begin{footnotesize}
\begin{alignat*}{3}
    E(w_{i+1}) - E(w_i) & = && (f(w_{i+1}) -f(w^*)) -  (f(w_{i}) -f(w^*)) && \\
    & = && f(w_{i+1}) -  f(w_{i}) && \\
    & = ^\dagger && f\left( w_i - \frac{\alpha}{L}\nabla f(w_i)\right) -  f(w_{i}) &&\\
    & \leq^\ddagger&&  \left[ f(w_i) - \frac{\alpha}{L}\langle \nabla f(w_i), \nabla f(w_i)\rangle \right. + \\ 
    &  && + \left. \frac{L}{2}\frac{\alpha^2}{L^2}\lVert \nabla f(w_i)\rVert_2^2\right] - f(w_i)  \\
    & = && -\frac{\alpha}{L}(1 -\frac{\alpha}{2}) \lVert \nabla f(w_i)\rVert_2^2 . &&
\end{alignat*}
\end{footnotesize}

Where for $\dagger$ we use the fact $w_{i+1}$ is obtained through the guarding mechanism, and for $\ddagger$ we applied inequality (\ref{lsmoothprop}) with $y=w_i - \frac{\alpha}{L}\nabla f(w_i)$ and $x = w_i$. This proves inequality (\ref{lsmootheq1}) if $w_{i+1}$ was chosen using the guarding.

If $w_{i+1}$ was chosen using L2O, then we know by criterion (\ref{lgl2ocriterion}) that $$f(w_{i+1}) < f\left( w_i - \frac{\alpha}{L}\nabla f(w_i)\right) $$ and so the above derivation follows through except that at $\dagger$ we have a (strict) inequality instead of an equality.  

Inequality (\ref{lsmootheq2}) is easily obtained from (\ref{lsmootheq1}) by noticing that $\forall w, E(w)\geq 0$ because $w^*$ is the global minimum of $f$, and so $- E(w_i) \leq E(w_{i+1}) - E(w_i)$.
\end{proof}

\begin{thm}
\label{thm:lyap}
Let $f$ be a continuous loss function which is strictly convex and $L$-smooth and let $w^*$ be it's global minimum. Then $E(w) := f(w) - f(w^*)$ is a Lyapunov function for sequences of points generated by the Loss-Guarded L2O algorithm using (deterministic) gradient descent for the guarding mechanism with a step size of $\frac{a}{L}$ with $\alpha \in ]0, 2[ $. 
\end{thm}
\begin{proof}
To show that $E$ is a Lyapunov function we need to show that:
\begin{enumerate}
    \item $E$ is continuous.
    \item $E(w) = 0$ if and only if $w=w^*$.
    \item $E(w) > 0$ if and only if $w\neq w^*$.
    \item $E(w_{i+1}) \leq E(w_i), \ \forall i \in \mathbb{N}$ .
\end{enumerate}

1. is automatic by continuity of $f$. 2 and 3 and also immediate by virtue of $w^*$ being the unique global minimum of $f$. And finally 4 is give from inequality (\ref{lsmootheq1}) of Proposition \ref{prop:lgl2olsmooth}. 
\end{proof}

\begin{lemma}
\label{lemma:SC-PLC}
Strong convexity implies the Polyak-Łojasiewicz Condition \citep{zhou2018fenchel}:
\begin{align}
    \label{PLC}
    E(w):= f(w) - f(w^*) \leq \frac{1}{2\mu} \lVert \nabla f(w)\rVert_2^2 .
\end{align}
\end{lemma}
\begin{proof}
From strong convexity we have that 
$$f(y) \geq f(x) + \langle \nabla f(x), y-x\rangle + \frac{\mu}{2}\lVert y - x\rVert_2^2 .$$
Using Hölder's inequality we have that
$$f(y) \geq f(x) - \lVert \nabla f(x) \rVert_2 \lVert y-x\rVert_2 + \frac{\mu}{2}\lVert y - x\rVert_2^2 .$$
Let $y_L$ be the value of $y$ which minimizes the left hand side of the inequality and $y_R$ be the value that minimizes the right hand side of the inequality. Then we have
\begin{align}
    f(y_L) & \geq f(x) - \lVert \nabla f(x) \rVert_2^2 \ \lVert y_L-x\rVert_2^2 + \frac{\mu}{2}\lVert y_L - x\rVert_2^2 \nonumber \\
    \label{eq:muconv2}
    & \geq f(x) - \lVert \nabla f(x) \rVert_2 \ \lVert y_R-x\rVert_2 + \frac{\mu}{2}\lVert y_R - x\rVert_2^2 .
\end{align}
Minimizing (\ref{eq:muconv2}) with respect to $y_R$ we find that $y_R$ satisfies:
\begin{align}
\label{eq:plc_prereq}
    \lVert y_R-x\rVert_2 = \frac{1}{\mu}\lVert \nabla f(x) \rVert_2 .
\end{align}

Now using the fact that by definition $y_L=x^*$, the minimum of $f$, and plugging (\ref{eq:plc_prereq}) in (\ref{eq:muconv2}) we get
\begin{align*}
    f(x^*) & \geq f(x) -  \frac{1}{2\mu}\lVert \nabla f(x) \rVert_2^2 ,
\end{align*}
which can be re-arranged to give the desired Polyak-Łojasiewicz Condition.
\end{proof}

\begin{thm}
\label{thm:conv}
Let $f$ be a continuous loss function which is $\mu$-strongly convex and $L$-smooth and let $w^*$ be it's global minimum. Then sequences of points generated by the Loss-Guarded L2O algorithm using (deterministic) gradient descent for the guarding mechanism with a step size of $\frac{a}{L}$ with $\alpha \in ]0, \min(2, 2 \mu L)[ $ converge to $w^*$, i.e. $$\substack{\lim \\ i\rightarrow \infty} w_i = w^*$$. 
\end{thm}
\begin{proof}
Let $E$ be the Lyapunov function of Theorem \ref{thm:lyap}. 
By $\mu$-strong convexity of $f$, we have that $f$ satisfies the Polyak-Łojasiewicz Condition (\ref{PLC}).
Thus from Proposition \ref{prop:lgl2olsmooth}, we have that:
\begin{align*}
E(w_{i+1}) - E(w_i) & \leq -\frac{\alpha}{L}(1 -\frac{\alpha}{2}) \lVert \nabla f(w_i)\rVert_2^2 \nonumber\\
 & \leq -\frac{\alpha}{2 \mu L}(1 -\frac{\alpha}{2}) E(w_i) .
\end{align*}
Where the second inequality follows from the Polyak-Łojasiewicz condition. Thus:
\begin{align}
\label{convineq}
    E(w_{i+1}) \leq  \left[1 -\frac{\alpha}{2 \mu L}(1 -\frac{\alpha}{2})\right] E(w_i) .
\end{align}
And thus for $\alpha \in ]0, \min(2, 2 \mu L)[$, convergence to $w^*$ is guaranteed because (\ref{convineq}) implies that
$$ E(w_{i}) \leq  \left[1 -\frac{\alpha}{2 \mu L}(1 -\frac{\alpha}{2})\right]^i E(w_0) ,$$
which means that $\substack{\lim \\ i\rightarrow \infty} E(w_i) = 0$.
\end{proof}

\begin{definition}
Given a loss function for $f: \R^k \rightarrow \R$ composed of the sum of the loss for each sample $x\in X$:
$$f(w) := \sum_{x \in X} \varphi(w, x) \mu_X(x) , $$
where $\varphi: \R^k \times X \rightarrow \R $ is the sample-dependent loss function and $\mu_X(x)$ is the probability density of sample $x$, we define the mini-batch stochastic gradient of mini-batch size M of $f$ at $w$, $\nabla_{mb}f(w)$, to be the random variable $\nabla_{mb}f(w): X^M \rightarrow \R^k$ with probability distribution
\begin{footnotesize}
\begin{align}
    P\left(\nabla_{mb}f(w) = \frac{1}{\sum_{i=1}^M\mu_X(x_i)} \sum_{i=1}^M \nabla_w \varphi(w, x_i) \mu_X(x_i) \right) = \nonumber \\ 
    = \prod_{i=1}^M \mu_X(x_i) \nonumber .
\end{align}
\end{footnotesize}

\end{definition}
\begin{prop}
\label{prop:sgdlsmooth}
Let $f$ be a continuous loss function which is $L$-smooth and let $w^*$ be it's global minimum. Let $E(w) := f(w) - f(w^*)$. Let $w_{i+1}$ and $w_{i}$ be two subsequent points in a sequence generated by the Loss-Guarded L2O algorithm using mini-batch stochastic gradient descent with a step size of $\alpha/L$ for the guarding mechanism. If, in expectation, the stochastic gradient $\nabla_{mb} f(w)$ is equal to the true gradient, 
$$\Esp(\nabla_{mb} f(w)) = \nabla f(w) , $$ and the variance of the stochastic gradient is bounded, 
$$\operatorname{Var}(\lVert \nabla_{mb} f(w)\lVert_2^2 ) \leq \sigma^2 ,$$ then:
\begin{align}
\label{sgdlsmootheq1}
    E(w_{i+1}) - E(w_i) \leq & -\frac{\alpha}{L}(1 -\frac{\alpha }{2}) \nonumber \\
    & \left( \lVert \nabla f(w_i)\rVert_2^2  - \frac{\alpha\sigma^2}{2(1-\alpha/2)}\right)
\end{align}
and 
\begin{align}
\label{sgdlsmootheq2}
    E(w_i) \geq \frac{\alpha}{L}(1 -\frac{\alpha }{2})\left( \lVert \nabla f(w_i)\rVert_2^2  - \frac{\alpha\sigma^2}{2(1-\alpha/2)}\right)
\end{align}
\end{prop}
\begin{proof}
The proof follows as in Proposition \ref{prop:lgl2olsmooth}, except that we now do it in expectation (over the mini-batch) and that if the guarding mechanism is used, then $w_{i+1} = w_i - \frac{\alpha}{L}\nabla_{mb}f(w_i)$, but as per criterion (\ref{lgl2ocriterion}), whether L2O is used or the  guarding mechanism, we have that $f(w_{i+1}) \leq f(w_i - \frac{\alpha}{L}\nabla_{mb}f(w_i))$. 
Thus assuming that $w_{i+1}$ is chosen via the guarding mechanism we have that
\pagebreak
\begin{strip}  
\begin{align*}
    \Esp\left[E(w_{i+1}) - E(w_i)\middle|w_i\right] &= \big(\Esp\left[f(w_{i+1})\middle|w_i\right] -f(w^*)\big) -  \big(f(w_{i}) -f(w^*)\big) \\
    & = \Esp\left[f(w_{i+1})\middle|w_i\right] -  f(w_{i}) \\
    & \leq \Esp\left[f\left( w_i - \frac{\alpha}{L}\nabla_{mb} f(w_i)\right)\middle|w_i\right] -  f(w_{i}) \\
    & \leq^\dagger \left\{ f(w_i) - \frac{\alpha}{L}\langle \nabla f(w_i), \Esp\left[\nabla_{mb} f(w_i)\middle|w_i\right]\rangle + \frac{L}{2}\frac{\alpha^2}{L^2}\Esp\left[\lVert \nabla_{mb} f(w_i)\rVert_2^2\middle|w_i\right]\right\} - f(w_i) \\
    & \leq^\ddagger \left\{ f(w_i) - \frac{\alpha}{L}\langle \nabla f(w_i), \nabla f(w_i)\rangle + \frac{L}{2}\frac{\alpha^2}{L^2}\left(\lVert \nabla f(w_i)\rVert_2^2 + \sigma^2\right)\right\} - f(w_i)\\
    & = -\frac{\alpha}{L}(1 -\frac{\alpha }{2})\left( \lVert \nabla f(w_i)\rVert_2^2  - \frac{\alpha\sigma^2}{2 (1-\alpha/2)}\right),
\end{align*}
\end{strip}
where for $\dagger$ we applied inequality (\ref{lsmoothprop}) with $y=w_i - \frac{\alpha}{L}\nabla_{mb} f(w_i)$ and $x = w_i$, and for $\ddagger$ we use the fact that $\Esp\left[\nabla_{mb} f(w_{i+1})\middle|w_i\right] = \nabla f(w)$ and that since the variance of $\nabla_{mb} f(w)$ is bounded by $\sigma^2$ we have that $\Esp\left[\lVert \nabla_{mb} f(w_i)\rVert_2^2\middle|w_i\right] \leq \lVert \nabla f(w_i)\rVert_2^2 + \sigma^2$. 

Again inequality (\ref{sgdlsmootheq2}) follows trivially from  (\ref{sgdlsmootheq1}) by noticing that $\forall w, E(w)\geq 0$ because $w^*$ be is the global minimum of $f$, and so $- E(w_i) \leq E(w_{i+1}) - E(w_i)$.
\end{proof}

\begin{thm}
\label{thm:sgdconv}
Let $f$ be a continuous loss function which is $\mu$-strongly convex and $L$-smooth and let $w^*$ be it's global minimum. If, in expectation, the stochastic gradient $\nabla_{mb} f(w)$ is equal to the true gradient, 
$$\Esp(\nabla_{mb} f(w)) = \nabla f(w) , $$ and the variance of the stochastic gradient is bounded, then sequences of points generated by the Loss-Guarded L2O algorithm using stochastic gradient descent for the guarding mechanism with a step size of $\frac{a_i}{L}$ with $\alpha_i \propto \frac{1}{i_0 +i}$ converge to $w^*$, i.e. $$\substack{\lim \\ i\rightarrow \infty} w_i = w^* .$$ 
\end{thm}
\begin{proof}
From Proposition \ref{prop:sgdlsmooth} and $\mu$-strong convexity (which gives us the Polyak-Łojasiewicz Condition (\ref{PLC})) we have that 
\begin{footnotesize}
\begin{align}
    \Esp\left[E(w_{i+1}) - E(w_i)\middle|w_i\right] & \leq  -\frac{\alpha}{2 \mu L}(1 -\frac{\alpha }{2}) E(w_i)  + \frac{\alpha^2\sigma^2}{2 L}.
\end{align}
\end{footnotesize}
Which we can rewrite as
\begin{align}
\label{eq:sgdrec}
    \Esp\left[E(w_{i+1}) \middle|w_i\right] & \leq  \left(1-\frac{\alpha}{2 \mu L}(1 -\frac{\alpha }{2})\right) E(w_i)  + \frac{\alpha^2\sigma^2}{2 L} \nonumber \\
    & \leq \left(1-\frac{\alpha}{4 \mu L}\right) E(w_i)  + \frac{\alpha^2\sigma^2}{2 L} ,
\end{align}
where the second inequality is valid if $0<\alpha<1$ because then $(1 -\frac{\alpha }{2})> \frac{1}{2}$.

We will now prove by recursion, using inequality (\ref{eq:sgdrec}) that with the appropriate learning rate $\lambda_i=\frac{\alpha_i}{L}$ we have the follow inequality for all $i$:
\begin{align}
    \label{eq:finalsgd}
    \Esp\left[E(w_{i+1}) \right] & \leq \frac{C}{i + i_0 + 1} , 
\end{align}
for positive constants $C = 32 \mu^2 L \sigma^2$ and $i_0$. Which implies that $\substack{\lim \\ i\rightarrow \infty}\Esp\left[E(w_{i+1}) \right]\rightarrow 0$ and thus that $w$ converges to $w^*$ in expectation. 

For $i=0$:
\begin{align*}
    \Esp\left[E(w_{0})\right] & \leq \frac{C}{i_0} = E(w_0) 
\end{align*}
is true if we set $i_0 = C/E(w_0)$.

Now supposing inequality (\ref{eq:finalsgd}) is true up to $i$, let us show it for $i+1$ by using $\alpha_i:= \min(1; 2\mu L; \frac{8\mu L}{i+ i_0})$:
\begin{align*}
    \Esp\left[E(w_{i+1})\right] & \leq \left(1-\frac{\alpha_i}{4 \mu L}\right) \Esp[E(w_i)]  + \frac{\alpha_i^2\sigma^2}{2 L} \\
    & \leq \left(1-\frac{\alpha_i}{4 \mu L}\right) \frac{C}{i + i_0}  + \frac{\alpha_i^2\sigma^2}{2 L} \\
    & \leq \left(1-\frac{2}{i + i_0}\right) \frac{C}{i + i_0}  + \frac{(8 \mu L)^2\sigma^2}{2 L(i+i_0)^2} \\
    & = \left(1-\frac{2}{i + i_0}\right) \frac{C}{i + i_0}  + \frac{C }{(i+i_0)^2} \\
    & = C\frac{i+i_0 - 1}{(i + i_0)^2}\\
    & \leq \frac{C}{i+1 + i_0}, 
\end{align*}
where the last inequality uses $\frac{k-1}{k^2}\leq \frac{1}{k+1}$
\end{proof}

\section{Experiments with Different Learned Optimizer}
\label{sec:different}

This paper tries to address a common problem of learned optimizers: they are not easily predictable and their behaviour can depend on the meta-testing procedure (e.g. on what optimizee and dataset they were meta-trained on). For this reason we show the same experiments as in section \ref{sec:experiments}, but with a different learned optimizer (L2O) that was meta-trained on a \textbf{ConvNet} optimizee and the \textbf{Cifar10} dataset instead of on an MLP with MNIST as in section \ref{sec:experiments}. The purpose of these experiments is to:
\begin{itemize}
\item show that our method (LGL2O) works independently of which learned optimizer is used and that
\item LGL2O acts appropriately in all situations and provides the best of both the learned optimizers and the general optimizer.
\end{itemize}

These experiments follows exactly the setup shown in the Experiments section \ref{sec:experiments}. First, the L2O optimizer was meta-trained to optimize the ConvNet optimizee on the Cifar10 dataset. All the experiment parameters and hyperparameters follow the original experiments exactly, except of the learning rate of the meta-optimizer set to $lr=0.0001$. No fine-tuning was made in order to obtain these results. The meta-training progress is shown in the Figure \ref{fig:ap:meta-learning}.

\begin{figure}[htb]
    \includegraphics[width=\columnwidth]{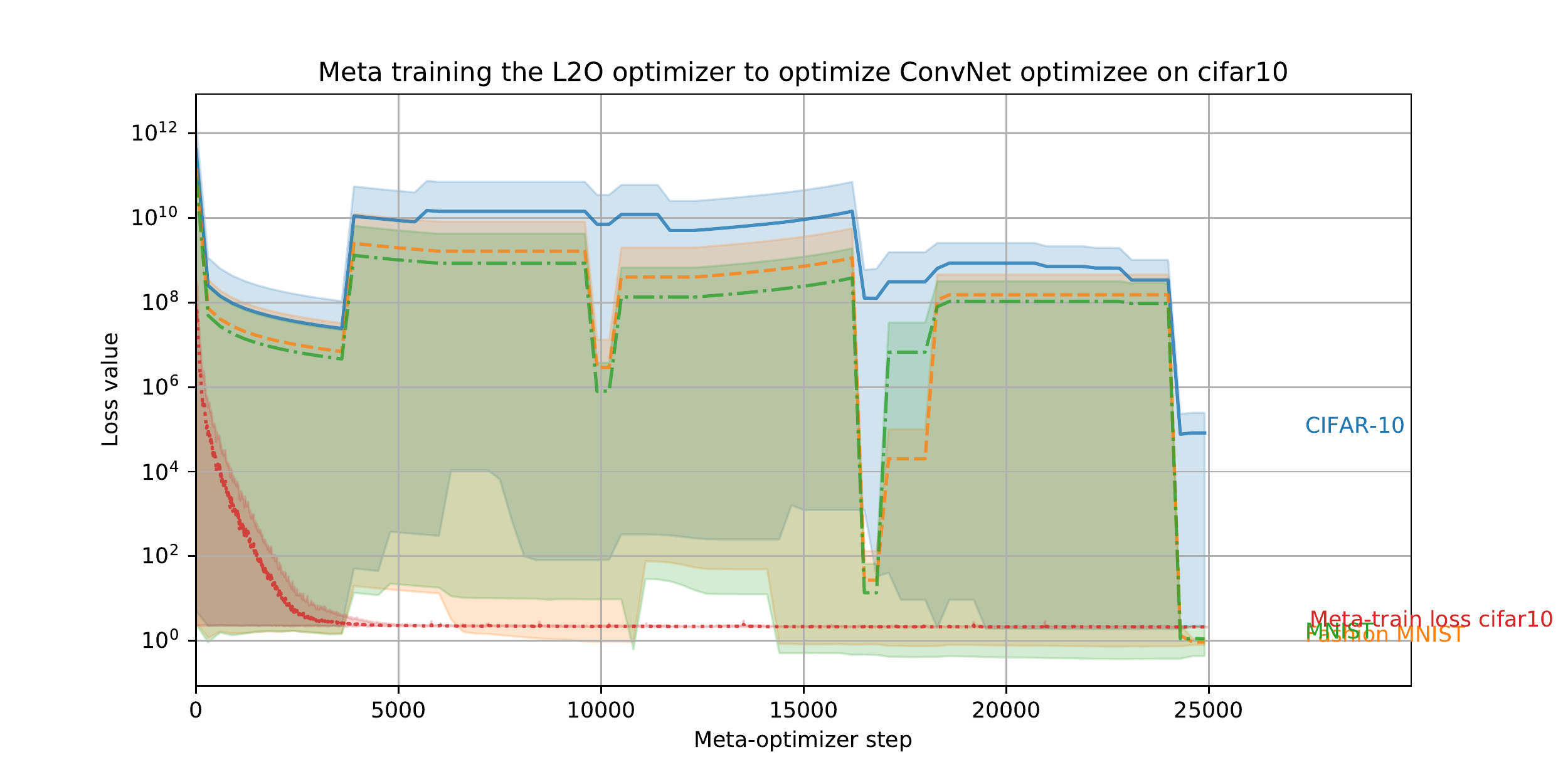}
    \caption{Outer-loop convergence of the learned optimizer on the Cifar10 dataset. The L2O method converges to sub-optimal, but stable results on the train set. During meta-testing, it is apparent that variance between test runs is high, but the optimizer performs comparably well on the test-set and generalizes the behaviour to other datasets. This meta-testing procedure shows only subset of datasets where ConvNet optimizee can be used. Mean and min-max ranges are across 5 runs are shown.}
    \label{fig:ap:meta-learning}
\end{figure}

For clarity, the Table \ref{table:ap:exp} shows the meta-testing experiment settings.

\begin{table}[h!]
\begin{tabular}{llll}
\toprule
\multicolumn{2}{l}{In Distribution?} & \multicolumn{2}{l}{Experiment Type} \\ \cmidrule(lr){1-2}\cmidrule(lr){3-4}
\textbf{Dataset}          & \textbf{Optimizee}         & \textbf{Dataset}               & \textbf{Optimizee}   \\ \toprule
no              & no               & MNIST                 & MLP         \\
no              & yes                & MNIST                 & Conv        \\
no               & no               & Spirals \& Circles    & MLP         \\
yes               & yes                & CIFAR-10              & Conv        \\ 
\end{tabular}
\caption{Overview of the experiments with updated in/out-distribution settings: after the meta-training phase, the ability of the hybrid optimizer LGL2O is evaluated on long rollouts, on in-distribution and out-of-distribution datasets and optimizees.}
\label{table:ap:exp}
\end{table}

\begin{figure}[htb]
    \centering
    \includegraphics[width=\columnwidth]{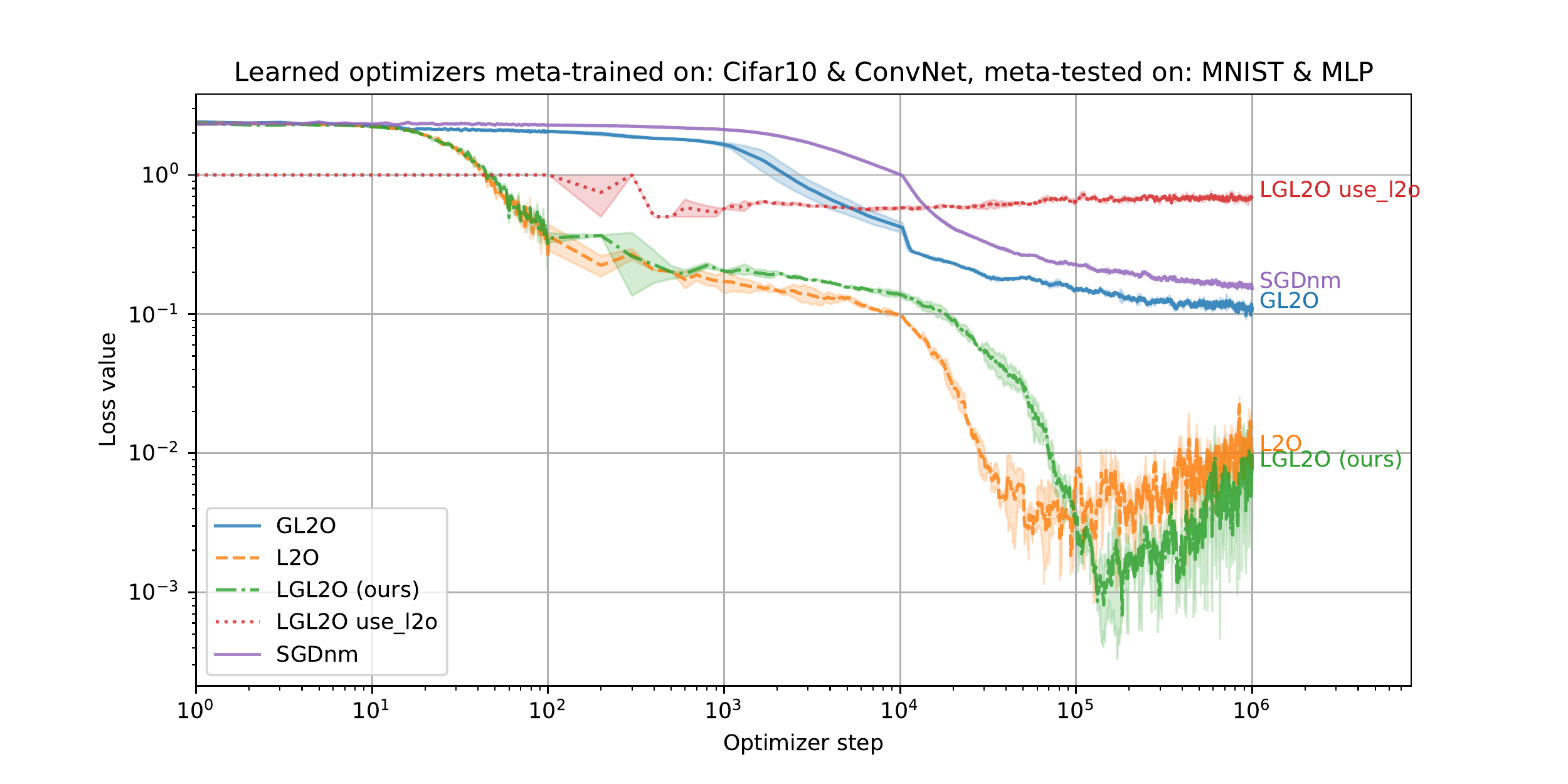}
    \caption{\textbf{Out-of-distribution dataset, out-of-distribution optimizee.} MLP optimizee optimized by different optimizers on the MNIST dataset. The L2O optimizer meta-trained on ConvNet produces very stable behaviour which starts to diverge after around $10^5$ steps. Our LGL2O optimizer is slightly slower, achieves better performance and then starts slowly diverging as well. This is probably caused by noise in the switching mechanism, since it can be seen that the $use\_l2o$ indicates that the guard is switching at high frequency.}
    \label{fig:ap:mnist}
\end{figure}

\begin{figure}[htb]
    \centering
    \includegraphics[width=\columnwidth]{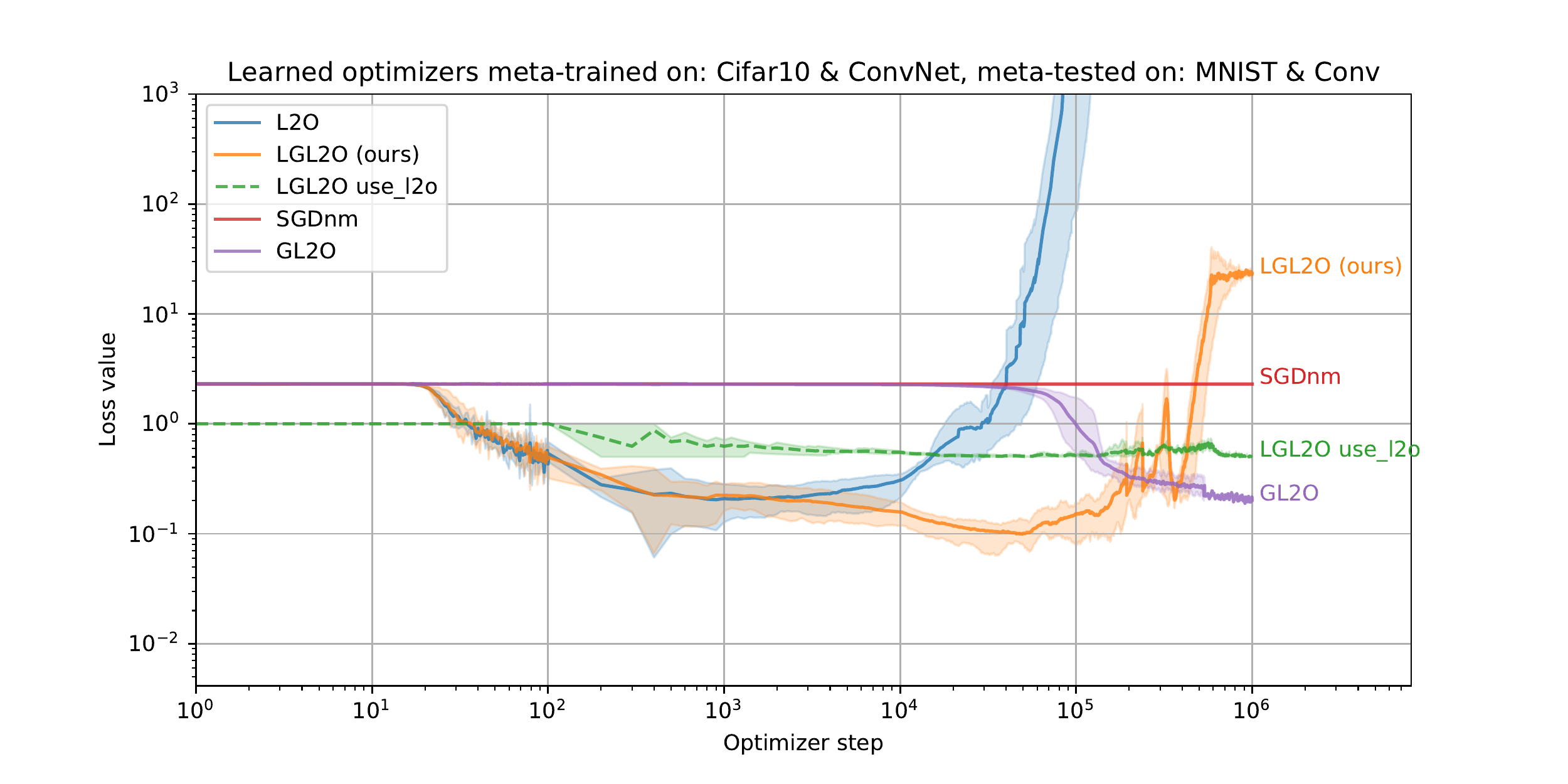}
    \caption{\textbf{Out-of-distribution dataset, in-of-distribution optimizee.} ConvNet optimizee optimized by different optimizers on the MNIST dataset. Here, the L2O diverges significantly, the GL2O counters the tendency of the learned optimizer to diverge. The LGL2O diverges significantly because criterion \ref{lgl2ocriterion} is not reliably evaluated anymore as can be seen by high frequency switching between learned optimizer and fallback optimizer after $10^5$ steps (green line). Close to the minimum of the loss function, 10 mini-batches is not enough anymore to reliably distinguish between the losses given by the updates proposed by the learned optimizer and the fallback optimizer. Until then however, LGL2O improved significantly on both it's constituent parts. }
    \label{fig:ap:mnist_conv}
\end{figure}

\begin{figure}[htb]
    \centering
    \includegraphics[width=\columnwidth]{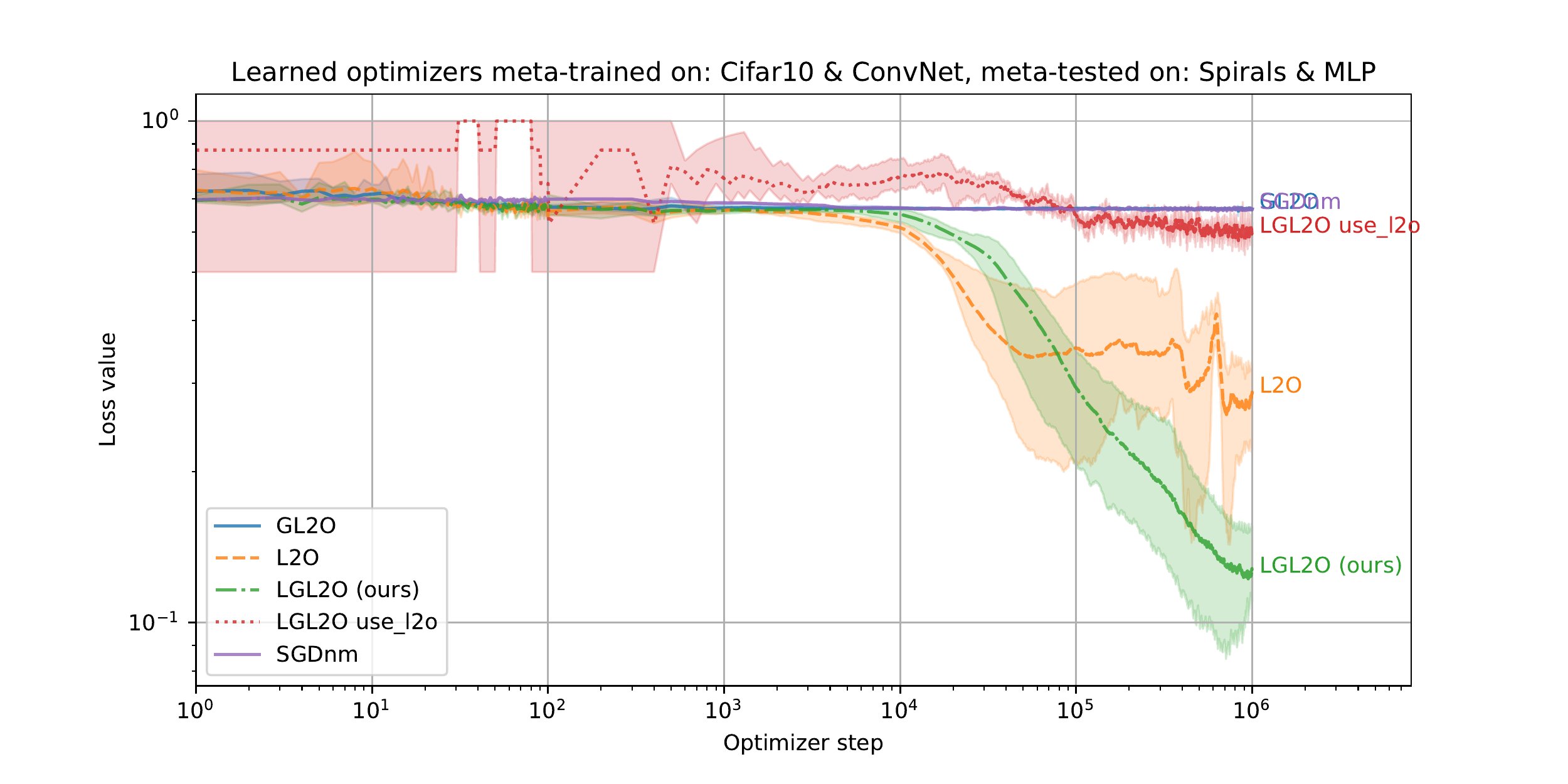}
    \caption{\textbf{Out-of-distribution dataset, out-of-distribution optimizee.} MLP optimizee optimized by different optimizers on the Spirals dataset. The Spirals dataset has a strong local optimum if used with a small MLP. In this case, the SGDnm and GL2O get stuck, while the L2O and LGL2O avoid the local optimum and converge well. In this case, our LGL2O outperforms L2O and keeps converging steadily.}
    \label{fig:ap:spirals_mlp}
\end{figure}

\begin{figure}[htb]
    \centering
    \includegraphics[width=\columnwidth]{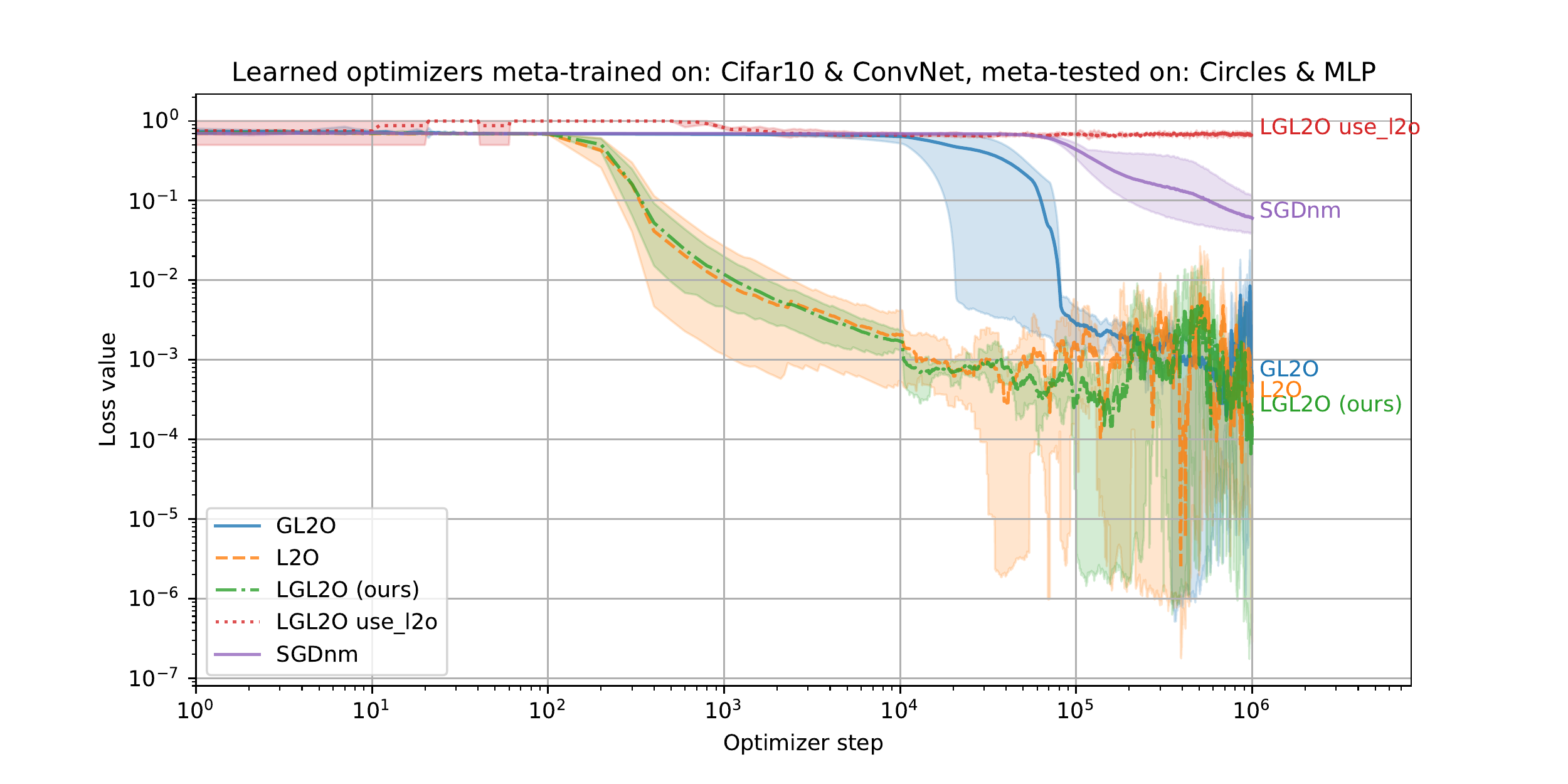}
    \caption{\textbf{Out-of-distribution dataset, out-of-distribution optimizee.} MLP optimizee optimized by different optimizers on the Circles dataset. In this case, all the learned optimizers are faster than the SGDnm baseline and they start to oscillate near optimal loss values. The L2O and LGL2O are faster than GL2O.}
    \label{fig:ap:circles_mlp}
\end{figure}

\begin{figure}[htb]
    \centering
    \includegraphics[width=\columnwidth]{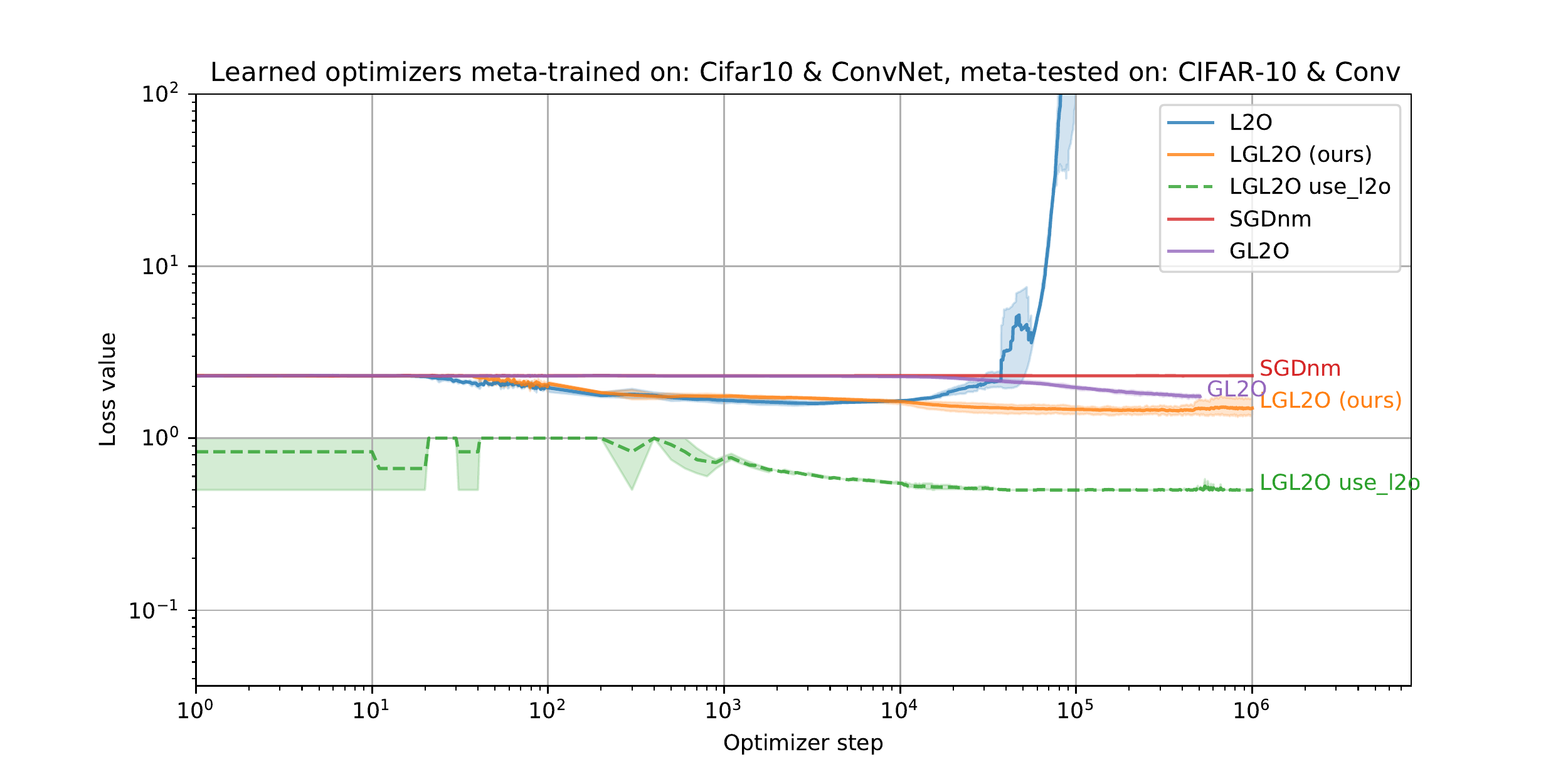}
    \caption{\textbf{In-distribution dataset, in-distribution optimizee.} ConvNet optimizee optimized by different optimizers on the Cifar10 dataset. Here, L2O optimizes well for many steps (as one would hope, being in distribution), but then it gets unstable. The GL2O deals with the instability and works as expected. The LGL2O works even better by benefiting from the excellent performance L2O and then slowly switching to it's fallback optimizer before L2O would start getting unstable. SGDnm, the fallback optimizer, in contrast has very poor performance as a standalone optimizer.}
    \label{fig:ap:cifar_conv}
\end{figure}

\subsection{Discussion}
\label{ap:discussion}

Here, the same meta-training procedure was used on a different optimizee and dataset. At the beginning of the meta-training (see Fig. \ref{fig:ap:meta-learning}), the L2O optimizer weights were initialized randomly. Probably due to the fact that the weights of ConvNet optimizee are shared, the implications of the weight modifications proposed by random L2O on the loss value are big and therefore the loss values explode in just first few optimizer steps. Gradually, the meta-training changes the optimizer to optimize the weights to a stable performance. Despite the fact that absolute values of the loss are rather far from optimal, it can be seen that the optimizer generalizes similar results to the in-distribution meta-test set and to other datasets. 

Then, the learned L2O optimizer\footnote{One of the meta-training runs that achieved good results on Cifar10 at the end.} parameters were used for meta-testing on long rollouts. 

It can be seen that the learned optimizer generalizes from the ConvNet optimizee to \textbf{MLP optimizee} very well. It does not tend to diverge much in very long rollouts, despite that it was trained only on rollouts that were 600 steps long. In these cases, our method (LGL2O) performs better or at least equally well compared to the L2O. It often outperforms the GL2O and it always outperforms the baseline SGNnm. This can be seen especially in the Fig. \ref{fig:ap:spirals_mlp}, where the SGDnm gets stuck in local optimum, L2O starts to plateau after $10^5$ steps, but our LGL2O switches to the fallback optimizer at this point to continue reducing the loss. LGL2O switched at the right time (i.e. after the local optimum was avoided) and continues converging at a faster rate than L2O after that. 

This is obvious in retrospect, but for LGL2O to work, the fallback optimizer cannot diverge. Initially we had used too large a learning-rate for the fallback optimizer which caused LGL2O to explode when it decided to switch the the fallback optimizer when the optimizee was a ConvNet \footnote{For this reason, graphs in Figures \ref{fig:ap:mnist_conv} and \ref{fig:ap:cifar_conv} use SGD $lr=0.0001$ instead of $lr=0.01$ and 100x slower learning rate decay.}. 

Of particular note is the divergence of LGL2O in Fig. \ref{fig:ap:mnist_conv}. This is exactly the risk of which we foretold in the last the paragraph of section \ref{sec:lossguard} and which me mentioned again in the paragraph on the Circles dataset in section \ref{sec:outdistribution}. Close to the minimum of the loss function, 10 mini-batches is not enough anymore to reliably distinguish between the losses given by the updates proposed by the learned optimizer and the fallback optimizer. We can see that that is the problem because of the high frequency switching between learned optimizer and fallback optimizer after $10^5$ steps (green line in Fig. \ref{fig:ap:mnist_conv}). As previously mentioned, at that point (past $10^5$ updates in this case) one would need to increase the hyperparameter $n_c$. Doing so in adaptive manner as suggested in section \ref{sec:outdistribution} would probably be best. But in all experiments in this paper, we did not hyperparameter search for the Guard hyperparameters and simply always used the same ones $n_t=n_c=10$.

\section{Hyperparameters}
\label{sec:hyperparams}

Hyperparameters for baseline learning algorithms (as well as guards) were partially chosen from experience and partially found by grid search. In all experiments, the guard has hyperparameters identical to the corresponding baseline - this is SGDnm in most cases. In all cases, our batch-size was 128 samples per mini-batch.

Hyperparameters of Adam were found for each combination of optimizee (MLP or ConvNet) and dataset found from the set [0.0001, 0.001, 0.01, 0.1] over 300 optimization steps. In case of SGD, the learning rate was set based on practical experience to 3.0 and (the optional) momentum to 0.9. 

In order to fulfill assumptions of convergence guarantee, the following decaying learning rate schedule was used:
\begin{align*}
    \text{lr}(t) = \frac{\text{lr}_0}{(\frac{t}{\text{decay-time-scale}} + 1)^{1.5}} , 
\end{align*}

where $t$ is the optimization step number and where the power of 1.5 was chosen because for convergence guarantees for SGD the power must be strictly greater than 0 and strictly less than 2, and 1.5 is a number in that range.

For SGDnm, grid search for each combination of optimizee and dataset was performed to determine the value of \textit{decay-time-scale} from the following set [2000, 5000, 10000, 20000, 50000], over 300 optimization steps, the same value of \textit{decay-time-scale} was used also for the SGD. LGL2O, the \textit{decay-time-scale} was set a constant value of 50'000 for all experiments without doing a grid search; that value was arrived at simply because SGD tended to plateau roughly at the $50'000$\textsuperscript{th} time step when running experiments at constant learning rates.

The best hyperparameters found using grid-search and then used as baselines (and optionally guards) were:

\begin{center}
\begin{tabular}{ |c|c|c| } 
 \hline
 & Adam lr & SGDnm decay \\ 
 \hline
 MLP MNIST & 0.001 & 50'000  \\ 
 \hline
 Conv MNIST & 0.01 & 20'000  \\ 
 \hline
  MLP Circles & 0.01 & 50'000 \\ 
 \hline
  MLP Spirals & 0.01 & 50'000 \\ 
  \hline
   Conv CIFAR10 & 0.001 & 20'000  \\ 
 \hline
\end{tabular}
\end{center}

The gradients of the L2O LSTM were truncated every 20 steps. In LGL2O 10 steps of L20 were compared with 10 steps of SGD on 10 "validation" mini-batches (pulled from the training data but different from the mini-batches on which the 20 steps were just made) and whichever did better on those 10 validation mini-batches got to update the optimizee. This helped avoiding effect of noise in the fitness evaluation and increased the stability of the LGL2O algorithm significantly. 

For GL2O we used $\alpha=0.99$ $\theta=0.9$, $m=30$, and used the exponentially moving average series.

\section{Randomness}
\label{sec:randomness}
Each experiment was run using 5 randomly chosen seeds. All neural network were initialized using the default PyTorch initializations. 

\section{Hardware and Software}
\label{sec:hardandsofware}
All experiments were code in Python 3.9 with PyTorch 1.8.1 on CUDA 11.0. Every run was run on a single NVIDIA GPU with a memory of between 4Gb and 12 Gb. Running one seed of all the optimization algorithms on one dataset takes roughly 5 hours of wallclock time depending on the specific GPU and usage level of the machine it is running on. The SciKit learn datasets were loaded from SciKit-Learn version 0.24.0. 

\section{Instability of Convergence}
\label{sec:stability}
This section is a supplementary section added after the publication of this paper to address an issue that was left unaddressed at publication time.

A few figures throughout the paper (specifically, Fig. \ref{fig:circles_mlp}, \ref{fig:ap:mnist}, and \ref{fig:ap:circles_mlp}) seem to show some instability of LGL2O at the very end of training. In section \ref{sec:outdistribution} and the legend of Fig. \ref{fig:circles_mlp}, we hypothesize that this is due to the approximation of the loss with only 10 batches not being good enough anymore when the the two points in criterion \ref{lgl2ocriterion} are close to enough to a local minimum and that thus increasing hyperparameters $n_t$ and $n_c$ would solve the instability. We have since checked this to see if this is indeed the case, and we see that this is correct. By increasing the hyperparameters $n_t$ and $n_c$, we regain stability at the end of training as can be seen in figures \ref{fig:circles_mlp_bb}, \ref{fig:ap:mnist_bb}, and \ref{fig:ap:circles_mlp_bb}.

\begin{figure}[htb]
    \centering
    \includegraphics[width=\columnwidth]{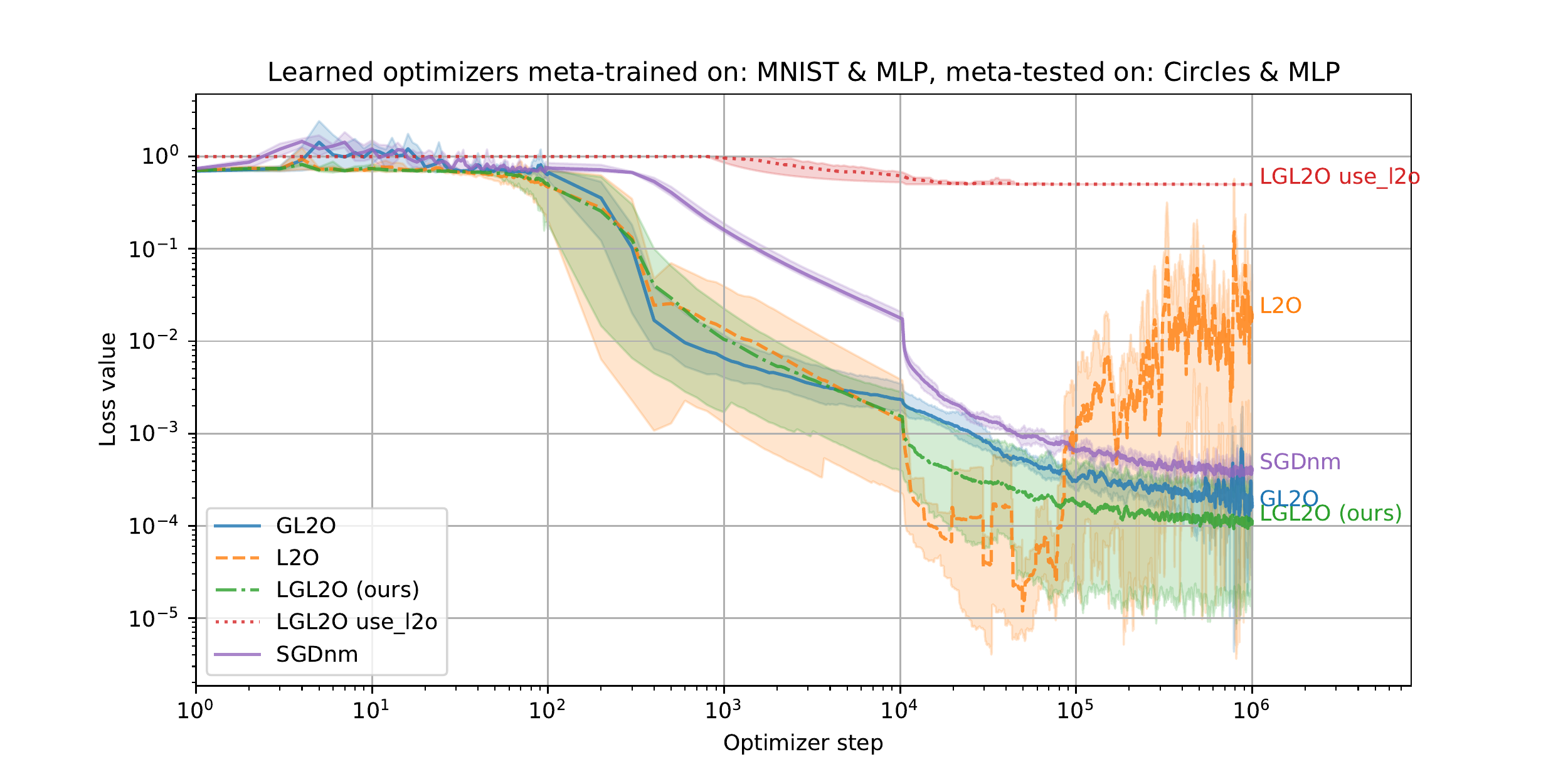}
\caption{\textbf{Out-of-distribution dataset, in-distribution optimizee.} Same experiment as in Fig. \ref{fig:circles_mlp} except that we set the hyperparameters $n_t$ = $n_c$ = 2 000 instead of 10 in the original. This removes instabilities.}
    \label{fig:circles_mlp_bb}
\end{figure}

\begin{figure}[htb]
    \centering
    \includegraphics[width=\columnwidth]{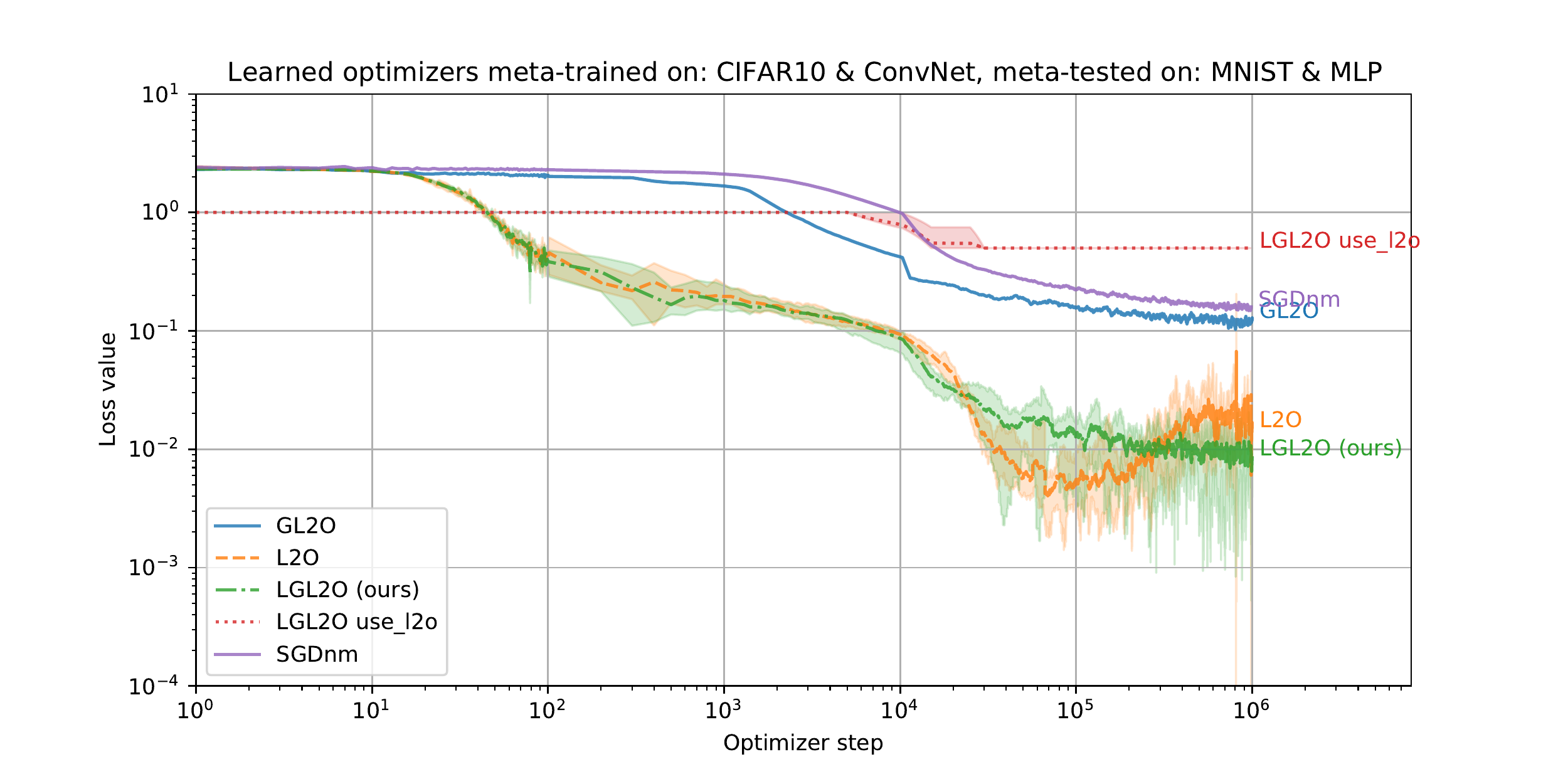}
    \caption{\textbf{Out-of-distribution dataset, out-of-distribution optimizee.} Same experiment as in Fig. \ref{fig:ap:mnist} except that we set the hyperparameters $n_t$ =5 000, $n_c$ = 2 000 instead of $n_t$ = $n_c$ = 10 in the original. This removes instabilities.}
    \label{fig:ap:mnist_bb}
\end{figure}

\begin{figure}[htb]
    \centering
    \includegraphics[width=\columnwidth]{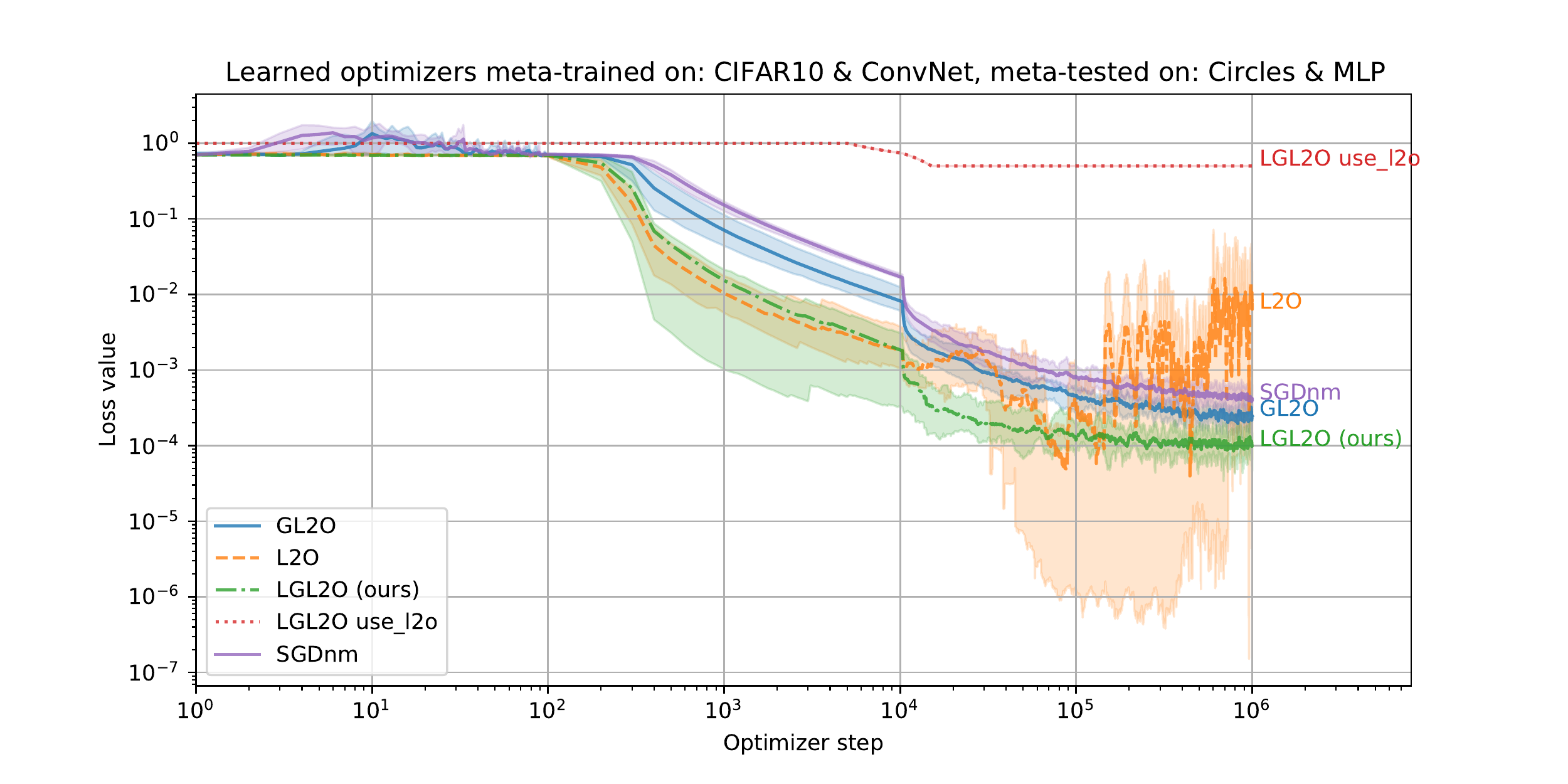}
    \caption{\textbf{Out-of-distribution dataset, out-of-distribution optimizee.} Same experiment as in Fig. \ref{fig:ap:circles_mlp} except that we set the hyperparameters $n_t$ =5 000, $n_c$ = 2 000 instead of $n_t$ = $n_c$ = 10 in the original. This removes instabilities.}
    \label{fig:ap:circles_mlp_bb}
\end{figure}
\end{document}